\documentclass[a4paper, 11pt]{article}
\usepackage{savesym}
\usepackage{amsmath,amssymb,amsfonts,amsthm}
\usepackage[numbers]{natbib}
\usepackage{algorithm}
\usepackage[noend]{algpseudocode}
\usepackage{comment}
\usepackage{xspace}
\usepackage{scalerel}
\usepackage{xcolor}
\usepackage[caption=false,font=footnotesize]{subfig}
\usepackage{multicol}
\usepackage[bookmarks=true]{hyperref}

\begin{document}

\title{New perspective on sampling-based motion planning via random
  geometric graphs}


\author{Kiril Solovey, Oren Salzman and Dan Halperin\thanks{ Kiril
    Solovey, Oren Salzman and Dan Halperin are with the Blavatnik
    School of Computer Science, Tel Aviv University, Israel;
    {\tt\footnotesize
      \{kirilsol,orensal,danha\}@post.tau.ac.il}.\newline \indent This
    work has been supported in part by the Israel Science Foundation
    (grant no.\ 1102/11), by the German-Israeli Foundation (grant no.\
    1150-82.6/2011), and by the Hermann Minkowski-Minerva Center for
    Geometry at Tel Aviv University. Kiril Solovey is also supported
    by the Clore Israel Foundation.}  }

\maketitle
\thispagestyle{empty}

\def\frechet{Fr\'echet\xspace}

\newcommand{\cupdot}{\mathbin{\mathaccent\cdot\cup}}

\newcommand{\mtm}{\emph{multi-to-multi}\xspace}
\newcommand{\mts}{\emph{multi-to-single}\xspace}
\newcommand{\sts}{\emph{multi-to-single-restricted}\xspace}
\newcommand{\dtd}{\emph{single-to-single}\xspace}

\newcommand{\cte}{\emph{full-to-edge}\xspace}
\newcommand{\ctc}{\emph{full-to-full}\xspace}
\newcommand{\ete}{\emph{edge-to-edge}\xspace}

\newcommand{\AND}{{\sc and}\xspace}
\newcommand{\OR}{{\sc or}\xspace}

\newcommand{\ignore}[1]{}

\def\vor{\text{Vor}}

\def\P{\mathcal{P}} \def\C{\mathcal{C}} \def\H{\mathcal{H}}
\def\F{\mathcal{F}} \def\U{\mathcal{U}} \def\L{\mathcal{L}}
\def\O{\mathcal{O}} \def\I{\mathcal{I}} \def\E{\mathcal{E}}
\def\S{\mathcal{S}} \def\G{\mathcal{G}} \def\Q{\mathcal{Q}}
\def\I{\mathcal{I}} \def\T{\mathcal{T}} \def\L{\mathcal{L}}
\def\N{\mathcal{N}} \def\V{\mathcal{V}} \def\B{\mathcal{B}}
\def\D{\mathcal{D}} \def\W{\mathcal{W}} \def\R{\mathcal{R}}
\def\M{\mathcal{M}} \def\X{\mathcal{X}} \def\A{\mathcal{A}}
\def\Y{\mathcal{Y}} \def\L{\mathcal{L}}

\def\dS{\mathbb{S}} \def\dT{\mathbb{T}} \def\dC{\mathbb{C}}
\def\dG{\mathbb{G}} \def\dD{\mathbb{D}} \def\dV{\mathbb{V}}
\def\dH{\mathbb{H}} \def\dN{\mathbb{N}} \def\dE{\mathbb{E}}
\def\dR{\mathbb{R}} \def\dM{\mathbb{M}} \def\dm{\mathbb{m}}
\def\dB{\mathbb{B}} \def\dI{\mathbb{I}} \def\dM{\mathbb{M}}

\def\eps{\varepsilon}

\def\limn{\lim_{n\rightarrow \infty}}

\def\orabf{\mathcal{O}_{B\!F}}
\def\orabfs{\mathcal{O}'_{B\!F}}
\def\orad{\mathcal{O}_{D}}

\def\obs{\mathrm{obs}}
\newcommand{\defeq}{%
  \mathrel{\vbox{\offinterlineskip\ialign{%
    \hfil##\hfil\cr
    $\scriptscriptstyle\triangle$\cr
    $=$\cr
}}}}
\def\Int{\mathrm{Int}}

\def\Reals{\mathbb{R}}
\def\Naturals{\mathbb{N}}
\renewcommand{\leq}{\leqslant}
\renewcommand{\geq}{\geqslant}
\newcommand{\compl}{\mathrm{Compl}}

\newcommand{\sig}{\text{sig}}

\newcommand{\sbs}{sampling-based\xspace}
\newcommand{\mr}{multi-robot\xspace}
\newcommand{\mpl}{motion planning\xspace}
\newcommand{\mrmp}{multi-robot motion planning\xspace}
\newcommand{\sr}{single-robot\xspace}
\newcommand{\cs}{configuration space\xspace}
\newcommand{\conf}{configuration\xspace}
\newcommand{\confs}{configurations\xspace}

\newcommand{\stl}{\textsc{Stl}\xspace}
\newcommand{\boost}{\textsc{Boost}\xspace}
\newcommand{\core}{\textsc{Core}\xspace}
\newcommand{\leda}{\textsc{Leda}\xspace}
\newcommand{\cgal}{\textsc{Cgal}\xspace}
\newcommand{\qt}{\textsc{Qt}\xspace}
\newcommand{\gmp}{\textsc{Gmp}\xspace}

\newcommand{\Cpp}{C\raise.08ex\hbox{\tt ++}\xspace}

\def\concept#1{\textsf{\it #1}}
\def\ccode#1{{\texttt{#1}}}

\newcommand{\ch}{\mathrm{ch}}
\newcommand{\pspace}{{\sc pspace}\xspace}
\newcommand{\threesum}{{\sc 3Sum}\xspace}
\newcommand{\np}{{\sc np}\xspace}
\newcommand{\degree}{\ensuremath{^\circ}}
\newcommand{\argmin}{\operatornamewithlimits{argmin}}

\newcommand{\Gdisk}{\G^\textup{disk}}
\newcommand{\Gbt}{\G^\textup{BT}}
\newcommand{\Gsoft}{\G^\textup{soft}}
\newcommand{\Gnear}{\G^\textup{near}}
\newcommand{\Gembed}{\G^\textup{embed}}

\newcommand{\dist}{\textup{dist}}

\newcommand{\Cfree}{\C_{\textup{free}}}
\newcommand{\Cforb}{\C_{\textup{forb}}}

\newtheorem{lemma}{Lemma}
\newtheorem{theorem}{Theorem}
\newtheorem{corollary}{Corollary}
\newtheorem{claim}{Claim}

\theoremstyle{definition}
\newtheorem{definition}{Definition}
\newtheorem{remark}{Remark}
\theoremstyle{plain}
\newtheorem{observation}{Observation}

\def\rfunc{\left(\frac{\log n}{n}\right)^{1/d}}
\def\rfuncs{\left(\frac{\log n}{n}\right)^{1/d}}
\def\cfunc{\sqrt{\frac{\log n}{\log\log n}}}
\def\rtrs{\gamma\rfunc}
\def\ctrs{2\cfunc}
\def\aconn{\A_\textup{conn}}
\def\abd{\A_\textup{str}}
\def\aspan{\A_\textup{span}}
\def\aopt{\A_\textup{opt}}
\def\ao{\A_\textup{ao}}
\def\acfo{\A_\textup{acfo}}
\def\binomial{\textup{Binomial}}
\def\twin{\textup{twin}}

\def\as{a.s.\xspace}

\def\distU#1{\|#1\|_{\G_n}^U}
\def\distW#1{\|#1\|_{\G_n}^W}

\def\tooth{\scalerel*{\includegraphics{./../fig/tooth}}{b}}

\makeatletter
\def\thmhead@plain#1#2#3{%
  \thmname{#1}\thmnumber{\@ifnotempty{#1}{ }\@upn{#2}}%
  \thmnote{ {\the\thm@notefont#3}}}
\let\thmhead\thmhead@plain
\makeatother

\def\todo#1{\textcolor{blue}{\textbf{TODO:} #1}}
\def\new#1{\textcolor{magenta}{#1}}
\def\old#1{\textcolor{red}{#1}}

\def\removed#1{\textcolor{green}{#1}}

\def\dx{\,\mathrm{d}x}
\def\dy{\,\mathrm{d}y}
\def\drho{\,\mathrm{d}\rho}


\begin{abstract}
  Roadmaps constructed by many sampling-based motion planners
  coincide, in the absence of obstacles, with standard models of
  random geometric graphs (RGGs). Those models have been studied for
  several decades and by now a rich body of literature exists
  analyzing various properties and types of RGGs. In their seminal
  work on optimal motion planning Karaman and Frazzoli~\cite{KF11}
  conjectured that a sampling-based planner has a certain property if
  the underlying RGG has this property as well. In this paper we
  settle this conjecture and leverage it for the development of a
  general framework for the analysis of sampling-based planners. Our
  framework, which we call \emph{localization-tessellation}, allows
  for easy transfer of arguments on RGGs from the free unit-hypercube
  to spaces punctured by obstacles, which are geometrically and
  topologically much more complex.  We demonstrate its power by
  providing alternative and (arguably) simple proofs for probabilistic
  completeness and asymptotic (near-)optimality of probabilistic
  roadmaps (PRMs). Furthermore, we introduce several variants of PRMs,
  analyze them using our framework, and discuss the implications of
  the analysis.
\end{abstract}


\section{Introduction}
\label{sec:intro}
Motion planning is a fundamental research area in robotics with
applications in diverse domains such as graphical animation, surgical
planning, computational biology and computer games. For a general
overview of the subject and its applications, see,
e.g.,~\cite{clhbkt05, Latombe91, lavalle06}.

The basic problem of motion planning is concerned with finding a
collision-free path for a robot in a \emph{workspace} cluttered with
static obstacles.  The spatial pose of the robot, or its
\emph{configuration}, is uniquely defined by its degrees of freedom
(DOF{s}). The set of all configurations $\C$ is termed the
\emph{configuration space} of the robot, and decomposes into the
disjoint sets of free and forbidden configurations, namely~$\F$
and~$\C\setminus \F$, respectively. Thus, given start and target
configurations, the problem can be restated as the task of finding a
continuous curve in~$\F$ connecting the two configurations. This can
be very challenging, as $\F$ can be exponentially complex (see,
e.g.,~\cite{Can88, Rei79, SolHal15}) in the number of DOFs.

The high computational complexity of exact solutions to motion
planning have led to the development of sampling-based planners.
These algorithms, which trade completeness with applicability in
practical settings, aim to capture the connectivity of~$\F$ in a graph
data structure, called a roadmap, by randomly sampling~$\C$.  Most of
the theoretical properties of these algorithms are stated in terms of
their \emph{asymptotic} behavior, i.e., assuming that the number of
samples is sufficiently large: The property of \emph{probabilistic
  completeness} indicates that a given algorithm will eventually find
a solution (if one exists); algorithms that are known to be
\emph{asymptotically optimal} also return a solution whose cost
converges to the optimum.

Interestingly, roadmaps constructed by many sampling-based planners
coincide, in the absence of obstacles, with standard models of random
geometric graphs (RGGs).  These models have been studied for several
decades and by now a rich body of literature exists analyzing various
properties and types of RGGs.  Indeed, in their seminal work on
optimal motion planning, Karaman and Frazzoli~\cite{KF11} observed
this relation. They employed techniques that were initially developed
for the anaylsis of RGGs to the study of sampling-based
planners. Subsequent proofs regarding completeness and optimality of
new planners (see, e.g.,~\cite{GSB15, JSCP15,SH14}) rely, to some
extent, on the proofs in~\cite{KF11}.  Karaman and Frazzoli
conjectured that a sampling-based planner possesses a certain property
if the underlying RGG has this property as well (see~\cite[Section
6]{KF11}). The validity of this conjecture, which is settled in this
paper, allows to import existing results on RGGs directly to the
corresponding sampling-based planners.  \vspace{3pt}

\noindent \textbf{Contribution.} We introduce the
\emph{localization-tessellation} framework for the analysis of
sampling-based algorithms in motion planning. Our framework
facilitates the extension of properties of RGGs to sampling-based
techniques in motion planning. This is done using conceptually simple
ideas and elementary tools in probability theory. The underlying
result of the framework is that RGGs demonstrate similar behavior in
the absence as well as in the presence of obstacles. The framework
consists of two main components. First we show through
\emph{localization} that RGGs maintain their properties in
arbitrarily-small neighborhoods. The \emph{tessellation} stage extends
these properties to complex domains which can be viewed as free spaces
of motion-planning problems. Namely, the configuration space punctured
by obstacles.

We demonstrate the power of the framework by providing conditions for
probabilistic completeness and asymptotic (near-)optimality of
Probabilistic Roadmaps (PRMs)~\cite{KSLO96}. Our proofs are (arguably)
much simpler than the original proofs of Karaman and
Frazzoli~\cite{KF11}.

Furthermore, we introduce several variants of PRMs, which perform
connections in a randomized fashion, and analyze them using our
framework. Using those variants we show that the standard PRM still
maintains its favorable properties even when implemented using
\emph{approximate} nearest-neighbor search queries. \vspace{3pt}

\noindent \textbf{Organization.} In Section~\ref{sec:related} we
review related work. In Section~\ref{sec:RGG} we provide formal
definitions of several types of RGGs and describe their properties,
which will be employed by our localization-tessellation framework. In
Section~\ref{sec:local} we describe the \emph{localization} component
of the framework, that is, we show that RGGs maintain a wide range of
their properties in arbitrarily-small neighborhoods. In
Section~\ref{sec:tessellation} we focus on the two specific properties
of connectivity and bounded stretch and show that they hold in general
domains via a \emph{tessellation} argument.  In Section~\ref{sec:mp}
we make the transition to motion planning: we describe several
planners---including the standard PRM---and study their asymptotic
behavior using the framework. In Section~\ref{sec:evaluation} we show
empirically that the theoretical results obtained by the framework
also hold in practice. We conclude the paper with a discussion and
state several future research directions
(Section~\ref{sec:discussion}).

\section{Related work}\label{sec:related}
We review related work in the area of sampling-based algorithms for
motion planning and random geometric graphs.

\subsection{Sampling-based motion planning}
Sampling-based algorithms, such as PRMs~\cite{KSLO96}, Expansive Space
Trees (EST)~\cite{HLM99} and Rapidly-exploring Random Trees
(RRT)~\cite{KL00}, as well as their many variants, have proven to be
effective tools for motion planning.  These algorithms, and others were 
shown to be probabilistically complete.  While this is a
desirable property of any algorithm, in certain applications stronger
guarantees are required.

In recent years we have seen an increasing interest in
\emph{high-quality}\footnote{Quality can be measured in terms of
  length, clearance, smoothness, energy, to mention a few
  criteria. However, in this paper we will restrict our focus to the
  standard length measure.}  motion planning.  The literature contains
many examples of planners that are shown empirically to produce
high-quality paths (for a partial list see~\cite{ABDJV98, GO07, LTA03,
  LSMK13, REH11, SLN00, US03}).  Unfortunately, they are not
  backed by rigorous proofs pertaining to the quality of the solution
  produced by the algorithm.  A complementary work proves that in
certain settings RRT can produce paths of arbitrarily-poor
quality~\cite{NRH10}.

In their seminal work, Karaman and Frazzoli~\cite{KF11} develop the
first rigorous analysis of quality in the setting of sampling-based
motion planning: They provide conditions under which existing planners
are not asymptotically optimal. More importantly, they introduce two
new variants of RRT and PRM, termed RRT* and PRM*, which are shown to
be asymptotically optimal, under the right choice of parameters.
Following this exposition, several asymptotically-optimal algorithms
have emerged (see e.g.,~\cite{APD11, AT13,GSB15, JSCP15, SH15}). To
reduce the running time of such algorithms several asymptotically
\emph{near} optimal planners have been suggested, which trade the
quality of the solution with speed of computation (see
e.g.,~\cite{DB14, LLB13, SH14, SSAH14}).

Although the focus of this paper is on the simplified ``geometric''
setting of motion planning, we mention that some planners can
cope with more complex robotic systems in which uncertainty and
physical constraints come into play (see, e.g.,~\cite{KarFra10,
  LadKav04b, LiLitBek14, PerETAL12, SucKav12, WebBer13,
  XieETAL15}). Some of these planners can also produce high-quality
paths.

\subsection{Random geometric graphs}
The study of random geometric graphs (RGGs) was initiated by
Gilbert~\cite{Gil61} who considered the following model: a collection
of points is sampled at random in a given subspace of $\dR^d$, and a
graph is formed by drawing edges between points that are closer than a
given $r>0$, called the \emph{connection radius}.

An immediate question that follows is for which values of $r$ the
graph is connected (with high probability). Several works have
addressed this question and showed that it is both necessary and
sufficient that the connection radius will be proportional to
$\left(\frac{\log n}{n}\right)^{1/d}$, where $n$ is the number of
points and the points are sampled from the unit hypercube $[0,1]^d$
(see, e.g.,~\cite{AppRus02, KozLotStu10, Pen97}). Penrose~\cite{Pen99}
established that connectivity occurs approximately when the graph has
no isolated vertices. The monograph~\cite{Pen03} of the same author on
this subject studies many more properties of RGGs, including vertex
degree, clique size and coloring. The reader is also referred to a
survey on the subject by Walters~\cite{Wal11}.

In recent years RGGs have attracted much attention as a tool for
modeling large-scale communication networks, and in particular sensor
networks: the vertices of the graph represent sensors and an edge is
drawn between two sensors that are in the communication range. Gupta
and Kumar used this analogy in order to deduce the transmission power
necessary for the network to be connected~\cite{GuKum99}. An important
parameter that arises in this context is the number of transmitters a
message has to traverse in order to establish a broadcast between two
given transmitters.  Several works have established that this
parameter is proportional to the Euclidean distance between the two
nodes (see, e.g.,~\cite{BEFSS10,DMPP14, EllMarYan07, TobSauSta13,
  MitPer12, MP10}).

Various alternative connection strategies for RGGs have been proposed
over the years, the most studied of which is the $k$-nearest model
(see, e.g.,~\cite{BagSoh08, BBSW09,XueK04}). More complex models
assign edges between vertices in a randomized fashion (see,
e.g.,~\cite{BroETAL14, FriPeg14, Pen13}).  Some models introduce an
ordering on the sampled points (see, e.g.,\cite{BhaRoy04,PenWad10a,
  PenWad10, SchTha14, Wade09}) which results in a directed graph that
resembles the RRT tree~\cite{KL00}.

\section{Preliminaries}
\label{sec:RGG}
We describe several models of random geometric graphs (RGGs) and
mention useful properties that will be used throughout the paper. When
possible, we follow the notation and conventions in the standard
literature of RGGs (see, e.g.,~\cite{Pen03}). Let
$\X_n=\{X_1,\ldots,X_n\}$ be $n$ points chosen independently and
uniformly at random from the Euclidean $d$-dimensional
cube~$[0,1]^d$. We assume that the dimension~$d$ of the domain is
fixed and greater than one.  Let $\|x-y\|_2$ denote the Euclidean
distance between two points $x,y \in \dR^d$ and $\theta_d$ denote the
Lebesgue measure of the unit ball in~$\dR^d$.  Finally, denote by
$\B_{r}(x)$ be the $d$-dimensional ball of radius $r>0$ centered at
$x\in \dR^d$ and $\B_{r}(\Gamma) = \bigcup_{x \in
  \Gamma}\B_{r}(x)$ for any $\Gamma \subseteq \dR^d$. 
  Similarly, given a curve $\sigma:[0,1]\rightarrow
\dR^d$ denote $\B_r(\sigma)=\bigcup_{\tau\in[0,1]}\B_r(\sigma(\tau))$.

Throughout the paper we will use the standard notation for asymptotic
bounds: Let $f = f(n), g = g(n)$ be two functions. The notation
$f=\omega(g)$ indicates that $\limn f/g\rightarrow \infty$, and
$f=o(g)$ indicates that $\limn f/g\rightarrow 0$.  Let
$A_1,A_2,\ldots$ be random variables in some probability space and let
$B$ be an event depending on $A_n$. We say that $B$ occurs
\emph{almost surely} (\as, in short) if $\limn\Pr[B(A_n)]=1$. Finally,
all logarithms are at base $e$. 
\begin{definition}[\cite{Pen03}]\label{def:rgg}
  Given $r_n \in \dR^+$, the \emph{random geometric graph} (RGG)
  $\Gdisk(\X_n;r_n)$ is an undirected graph with the vertex set
  $ \X_n$.  For any two given vertices $x,y\in \X_n$ the graph
  contains the edge $(x,y)$ if $\|x-y\|_2\leq r_n$.
\end{definition}

We use the term RGG to refer both to the family of random geometric
graphs and to the specific model described in
Definition~\ref{def:rgg}.  This slight abuse of notation is introduced
to be consistent with existing literature and the exact meaning of RGG
will be clear from the context.

The following definition is concerned with a more complex structures
called \emph{random Bluetooth graphs}, also known as \emph{random
  irrigation graphs}.

\begin{definition}[\cite{BroETAL14}]\label{def:rbg}
  Let $2\leq c_n\leq n$ be a positive integer and $r_n \in \dR^+$.
  The random \emph{Bluetooth} graph (RBG) $\Gbt_n=\Gbt(\X_n;r_n;c_n)$
  is an undirected graph with the vertex set $\X_n$.  For every $x \in \X_n$ let
  $E(x, r_n)$ denote the set of points within maximal distance $r_n$ from
  $x$, i.e.,
  $E(x, r_n) = \{(x,y) : y\in \X_n \setminus \{x\}, \|x-y\|_2\leq
  r_n\}$.
  For every $x\in \X_n$ we pick randomly and independently $c_n$ edges
  from $E(x, r_n)$, and denote this set of edges by $E(x, r_n, c_n)$.
  The edge set of $\Gbt_n$ is defined to be
  $\bigcup_{x\in \X_n}E(x, r_n ,c_n)$.
\end{definition}

The following model is also a generalization of RGGs. Here a pair of
vertices are connected by an edge with a probability that depends on
the length of the edge.
\begin{definition}[\cite{Pen13}]\label{def:srbg}
  Let $r_n \in \dR^+$, and $\phi_n:\dR^+\rightarrow [0,1]$. The
  \emph{soft} random geometric graph (SRGG)
  $\Gsoft(\X_n;r_n;\phi_n)$ is an undirected graph with the
  vertex set $\X_n$. Denote by $E$ the edge set of this graph. For a
  pair of vertices $x,y\in \X_n$ such that $\|x-y\|_2\leq r_n$ it
  holds that $\Pr[(x,y)\in E]= \phi_n(\|x-y\|_2)$, independently for each
  edge.
\end{definition}

The following model can be viewed as a special case of SRGG where
$r_n=\infty$ and $\phi_n$ is constant.

\begin{definition}[\cite{FriPeg14}]\label{def:embedded}
  The randomly-\emph{embedded} geometric graph (REGG)
  $\Gembed(\X_n;p_n)$ is an undirected graph with the vertex set
  $\X_n$. For every two distinct vertices $x,y\in \X_n$, the graph
  contains the edge $(x,y)$ with probability $p_n$, and independently
  from the other edges.
\end{definition}

Throughout the text we will omit the superscript indicating the graph
type, and use instead the notation $\G_n$, if the exact type in
question is clear from the context.

\subsection{Connectivity}\label{sec:rgg_connectivity}
Recall that for every undirected graph~$\G$, two vertices~$u$ and~$v$
are called \emph{connected} if~$\G$ contains a path from~$u$ to~$v$.
A graph is said to be connected if every pair of its vertices is
connected.  We mention three results related to the connectivity of
the RGG,RBG, and SRGG models.
\begin{theorem}[\cite{BroETAL14}]\label{thm:rgg_con}
  Let $\G_n=\Gdisk(\X_n,r_n)$ and $r_n=\rtrs$.   Then
  $$\limn\Pr[\G_n\textup{ is connected}]=
  \begin{cases}
    0  & \quad \textup{if }\gamma<\gamma^*,\\
    1  & \quad \textup{if }\gamma>\gamma^*,\\
  \end{cases}$$
  where $\gamma^*=2(2d\theta_d)^{-1/d}$.
\end{theorem}

\begin{theorem}[\cite{BroETAL14}]\label{thm:rbg_con}
  Let $\G_n=\Gbt(\X_n;r_n;c_n)$, where $d\geq 2$, and $r_n=\rtrs$,
  where $\gamma >\gamma^{**}$ for $\gamma^{**}=d\cdot 2^{1+1/d}$.
  Then
$$\limn\Pr[\G_n\textup{ is connected}]=
  \begin{cases}
    0  & \quad \textup{if }c_n <  c_n^*,\\
    1  & \quad \textup{if }c_n > c_n^*,\\
  \end{cases}$$
where $c^*_n=\sqrt{\frac{2\log n}{\log \log n}}$.
\end{theorem}

Recently, Penrose~\cite{Pen13} developed a general characterization of
the necessary condition over $r_n$ and $\phi_n$ so that
$\Gsoft(\X_n;r_n;\phi_n)$ will be connected. We chose to focus here on
a specific range of values which can be of interest to motion
planning. The following theorem is proven in the appendix. 

\begin{theorem}\label{thm:srgg_con}
  Let $\G_n=\Gsoft(\X_n;r_n;\phi_n)$ and $r_n=\rtrs$. Set
  $\gamma > (d+1)^{1/d} \gamma^*$ (see Theorem~\ref{thm:rgg_con}), and
  define $\phi_n(z)=1-z/r_n$, for any $z\in \dR^+$. Then $\G_n$ is
  connected \as.
\end{theorem}

\subsection{Bounded stretch}\label{sec:rgg_distance}
Let $\G$ be a graph whose vertices are embedded in $\dR^d$. For every
two vertices $x,y\in \G$ denote their \emph{weighted graph distance},
i.e., the sum of lengths of the shortest path from $x$ to $y$, by
$\textup{dist}(\G,x,y)$.  Throughout the paper we will use the term
\emph{stretch} to denote the ratio between $\textup{dist}(\G,x,y)$ and
the length of the shortest path between $x,y$ in the domain in which
the graph is embedded. For instance, if this domain is convex, then
for every $x,y\in \G$ the stretch is defined to be
$\textup{dist}(\G,x,y)/\|x-y\|_2$.

In the setting of motion planning, we will
use the graph distance to bound the asymptotic path length of
sampling-based planners.

\begin{theorem}[\cite{TobSauSta13}]\label{thm:span2}
  Let $\G_n=\Gdisk(\X_n;r_n)$ with $r_n = \rtrs$ where
  $\gamma > \gamma^*$ (see Theorem~\ref{thm:rgg_con}).  Then there
  exists a constant $\zeta$ such that for every two vertices $x, y$ in
  the same connected component of $\G_n$, with
  $ \|x - y\|_2 = \omega(r_n)$, it holds that $\dist(\G_n,x,y)$ is at
  most $\zeta \|x-y\|_2$ \as
\end{theorem}

\begin{theorem}[\cite{FriPeg14}]\label{thm:span3}
  Let $\G_n=\Gembed(\X_n;p_n)$ and
  $p_n=\omega\left(\frac{\log^d n}{n}\right)$. Then for every two
  vertices $x,y\in \X_n$ it holds that $\dist(\G_n,x,y)$ is at most
  $\|x-y\|_2+o(1)$ \as
\end{theorem}

Notice that this statement does not condition the existence of a short
graph path on the event that the two vertices are in the same
connected component of the graph. Due to this fact it also follows
that the graph is connected with high probability.


\section{Localization of monotone properties of RGGs}\label{sec:local}
In this section we discuss graph properties and their asymptotic
behavior, when focusing on a subset of the domain $[0,1]^d$.  A
property $\A$ is \emph{monotone} if for every $G=(V,E)$ and $H=(V,E')$
such that $E\subseteq E'$, it holds that $G\in \A \implies H\in \A$.
Note that connectivity (Section~\ref{sec:rgg_connectivity}) and
bounded stretch (Section~\ref{sec:rgg_distance}) are
monotone.\footnote{ Additional examples of monotone properties for a
  graph~$\G$ are: $\G$ is Hamiltonian, $\G$ contains a clique of
  size~$t$, $\G$ is not planar, the clique number of $\G$ is larger
  than that of its complement, the diameter of $\G$ is at most $s$,
  etc.}

We now proceed to state several lemmas, which are the main ingredients
that allow us to extend existing properties of random graphs to more
complex domains than $[0,1]^d$ (see Section~\ref{sec:tessellation}).
The following lemma states that if an RGG \as possesses a certain
monotone property, then the restriction of this to a local domain \as
has the aforementioned property as well.

\begin{definition}\label{def:subgraph}
  Let $\G=(X,E)$ be a graph embedded in $[0,1]^d$, i.e., the vertices
  of $X$ represent points in $[0,1]^d$ and edges represent
  straight-line paths between the corresponding vertices. Given
  $\Gamma\subset [0,1]^d$ we denote by~$\G(\Gamma)$ the graph obtained
  from the intersection of $\G$ and $\Gamma$. This graph consists of
  the vertex set $X\cap \Gamma$ and all the edges in $E$ that are
  fully contained in $\Gamma$.
\end{definition}

\begin{definition}\label{def:localizable}
  Let~$\G_n$ be an RGG, RBG, SRGG or an REGG, defined over the vertex
  set~$\X_n$.  Then~$\G_n$ is \emph{localizable} for a property~$\A$
  if for every constant $0<\eps\leq 1$ and every $d$-dimensional
  axis-aligned cube~$B_{\eps}\subseteq [0,1]^d$ with side length of
  $\eps$ it holds that $\G_n(B_{\eps})\in \A$ \as
\end{definition}
\begin{lemma}\label{lem:local_rgg}
  Let $\A$ be a monotone property and $\gamma_{\A}$ some constant.
  Let $\G_n=\Gdisk(\X_n;r_n)$ be an RGG such that $\G_n\in \A$ \as,
  for $r_n=\gamma\left(\frac{\log n}{n}\right)^{1/d}$, where
  $\gamma>\gamma_{\A}$.  Then $\G_n$ is localizable for $\A$.
\end{lemma}
\begin{proof}
  For simplicity of presentation, we will use the notation $\G$ to
  refer to $\Gdisk$ throughout the proof.  Recall that $\X_n$ is a
  collection of~$n$ points chosen independently and uniformly at
  random from $[0,1]^d$.  We will also use
  $\Y_m^{\eps}=\{Y_1,\ldots,Y_m\}$ to denote a collection of~$m$
  points chosen independently and uniformly at random from
  $B_{\eps}$. Without loss of generality, assume that
  $B_{\eps}=[0,\eps]^d$.
  
  Observe that there exists a constant $\alpha>1$ independent of~$n$
  such that $\G(\X_n;r'_n)\in \A$ \as, where $r_n=\alpha r'_n$. The
  role of~$\alpha$ is purely technical and will become clear shortly.
  Now,
  \begin{align*}
    \Pr &[\G(\X_n \cap B_{\eps}; r_n) \notin \A]  =  \Pr 	[\G(\X_n \cap B_{\eps}; \alpha r'_n) \notin \A] \\
        &= \sum_{m=0}^{n}\Pr\left[\G\left(\X_n \cap B_{\eps}; \alpha r_n'\right)\not\in\A
          \Bigm|
          \left|\X_n \cap B_{\eps}\right|=m\right]  \cdot
          \Pr\left[|\X_n \cap B_{\eps}|=m\right] \\
        & = \sum_{m=0}^{{n}}\Pr\left[\G\left(\Y_m^{\eps}; \alpha r_n'\right)\not\in\A\right]\cdot
          \Pr\left[|\X_n \cap B_{\eps}|=m\right].
  \end{align*}

  Denote
 $$\sigma(i,j) = \sum_{m=i}^j\Pr\left[\G\left(\Y_m^{\eps}; \alpha r_n'\right)\not\in\A\right]\cdot
 \Pr\left[|\X_n \cap B_{\eps}|=m\right],$$
 and by definition we have that for $1\leq \ell \leq n$
$$\Pr 	\left[\G(\X_n \cap B_{\eps}; \alpha r_n') \notin \A \right] = \sigma(0,n) = \sigma(0,\ell-1) +  \sigma(\ell,n).$$
We show that for $\ell=\alpha^{-d}\eps^d n$ both
$\lim_{n\rightarrow \infty} \sigma(0,\ell-1) = 0$ and
$\lim_{n\rightarrow \infty} \sigma(\ell,n) = 0$ which will conclude
the proof of the lemma (for simplicity we assume that $\ell \in \dN$)
. We start with the former expression:
\begin{align*}
  \sigma(0,\ell-1)  
  &	= \sum_{m=0}^{{\ell-1}}\Pr\left[\G\left(\Y_m^{\eps}; \alpha r_n'\right)\not\in\A\right]\cdot
    \Pr\left[|\X_n \cap B_{\eps}|=m\right] \\
  &	\leq \sum_{m=0}^{\ell-1} 1\cdot \Pr\left[|\X_n \cap B_{\eps}|=m\right] \\
  & = \Pr \left[|\X_n \cap B_{\eps}| < \ell \right] \\
  & 	= \Pr \left[|\X_n \cap B_{\eps}| < \alpha^{-d} \dE[|\X_n \cap B_{\eps}|]\right] \\ 
  & 	\leq \exp\left\{-n\eps^d(1-\alpha^{-d})^2\right\}.
\end{align*}
The last inequality follows from the fact that $|\X_n\cap B_{\eps}|$
is a \emph{binomial random variable} with $n$ elements, success rate
of $|B_{\eps}|=\eps^d$ per trial, and a mean value of
$\dE[|\X_{n}(B_{\eps}|] = n\eps^d$. This in turn, enables the use of
\emph{Chernoff inequality} (see, e.g.,~\cite[Theorem 1.1]{DubPan09}),
the application of which is made possible due to the $\alpha^{-d}$
factor.

We now focus on showing that
$\lim_{n\rightarrow \infty} \sigma(\ell,n) = 0$.  For any two integers
$n,m$ such that $\ell\leq m\leq n $ we have that
  \begin{align*}
    \Pr[\G(\Y_m^{\eps}; \alpha r_n')\notin \A] 
    & \stackrel{(1)}{=}\Pr[\G(\X_m; \alpha \eps^{-1} r_n' )\notin \A]\\ 
    & \stackrel{(2)}{\leq} \Pr[\G(\X_m; r_m')\notin \A].
\end{align*}  
where the transitions are made possible due to (1)~a scaling of the
graph from $[0,\eps]^d$ to $[0,1]^d$; (2)~the monotonicity of $\A$ and
the fact that $r_m' \leq \frac{\alpha}{\eps} r_n' $. To show that
indeed $r_m' \leq \frac{\alpha}{\eps} r_n' $, note that
$\alpha\eps^{-1} >1 $ and that $\limn r_n'=0$. Thus,
\begin{align*}
  r_m' &
         \leq 	r_{\ell}' = \alpha^{-1} r_n   
         =		\alpha^{-1} \gamma \left(\frac{\log
         \alpha^{-d}\eps^{d} n}{\alpha^{-d}\eps^{d} n}\right)^{1/d} 
  \\ &=		\eps^{-1} \gamma \left(\frac{\log
       \alpha^{-d} \eps^{d} n}{n}\right)^{1/d}  \leq 	\eps^{-1}\gamma \left(\frac{\log n}{n}\right)^{1/d}
       =		\frac{\alpha}{\eps}  r_n'.
\end{align*}
Furthermore, set
$m^* = \text{argmax}_{m \in [\ell, n]} \left( \Pr \left[ \G(\X_m,r'_m)
    \notin \A \right] \right)$. It follows that
\begin{align*}
  \sigma(\ell,n)  
  & 	= \sum_{m=\ell}^{{n}}\Pr\left[\G\left(\Y_m^{\eps}; \alpha
    r_n'\right)\not\in\A\right]  \cdot
    \Pr\left[|\X_n \cap B_{\eps}|=m\right] \\
  &	\leq \sum_{m=\ell}^{{n}}\Pr \left[ \G(\X_m,r'_m) \notin \A \right] \cdot
    \Pr\left[|\X_n \cap B_{\eps}|=m\right] \\
  &	\leq  \Pr \left[ \G(\X_{m^*},r'_{m^*}) \notin \A \right]\sum_{m=\ell}^{{n}} \Pr\left[|\X_n \cap B_{\eps}|=m\right] \\        
  &	= \Pr \left[\G(\X_{m^*},r'_{m^*}) \notin \A \right] \cdot \Pr\left[|\X_n \cap B_{\eps}| \geq \ell \right]  \\
  &	\leq \Pr \left[ \G(\X_{m^*},r'_{m^*}) \notin \A \right].
\end{align*}
Note that $\limn\Pr\left[\G(\X_{m^*},r'_{m^*})\not\in\A\right]=0$,
which concludes the proof. %
\end{proof}

The following are the RBG, SRGG and REGG equivalents of
Lemma~\ref{lem:local_rgg}. 

\begin{lemma}\label{lem:local_rbg}
  Let $\A$ be a monotone property and $\gamma_{\A}$ some constant. Let
  $\G_n=\Gbt(\X_n;r_n;c_n)$ be an RBG such $\G_n\in \A$ \as, for every
  $r_n=\gamma_{\A}\left(\frac{\log n}{n}\right)^{1/d}$, where
  $\gamma>\gamma_{\A}$, and $c_n$ is non-decreasing. Then $\G_n$ is
  localizable for $\A$.
\end{lemma}

\begin{proof}
  We only prove the following inequality, as the rest of the proof
  proceeds in a manner similar to that of
  Lemma~\ref{lem:local_rgg}. We keep the notation from the previous
  proof. Recall that $\ell=\alpha^{-d}\eps^d n$. For any two integers
  $m, n$ such that $\ell\leq m\leq n $ we have that
  \begin{align*}
    \Pr[\G(\Y_m^{\eps}; \alpha r_n';c_n)\notin \A] 
    & =\Pr[\G(\X_m; \alpha \eps^{-1} r_n';c_n)\notin \A]\\ 
    & \leq \Pr[\G(\X_m; r_m';c_n)\notin \A] \\
    &\leq \Pr[\G(\X_m; r_m';c_m)\notin \A].
  \end{align*}
  Here we used the fact that $c_n\geq c_m$.
\end{proof}

\begin{lemma}\label{lem:local_srgg}
  Let $\A$ be a monotone property and let $\gamma_{\A}$ some constant.
Let 
  $\G_n=\Gsoft(\X_n;r_n;\phi_n)$ be an SRGG such that $\G_n\in \A$
  \as, where $r_n=\gamma\left(\frac{\log n}{n}\right)^{1/d}$ for some
  $\gamma>\gamma_{\A}$, and for every $z\in \dR^{+}$, the function
  $\phi_n(z)$ is increasing. Then $\G_n$ is localizable for~$\A$.
\end{lemma}

\begin{proof}
  The proof is identical to that of Lemma~\ref{lem:local_rbg}. One
  only needs to replace $c_n$ with $\phi_n$.
\end{proof}

\begin{lemma}\label{lem:local_regg}
  Let $\A$ be a monotone property and let $\G_n=\Gembed(\X_n;p_n)$ be
  an REGG such that $\G_n\in \A$ \as, where~$p_n$ is
  non-decreasing. Then $\G_n$ is localizable for~$\A$.
\end{lemma}

\begin{proof}
  The proof follows very similar lines of the proof of
  Lemma~\ref{lem:local_rgg}. The main observation here is that for
  every $m<n$ it follows that
  $$\Pr[\G(\X_m;p_n)\not\in \A]\leq \Pr[\G(\X_m;p_m)\not\in \A]$$
  due to the monotonicity of $\A$ and the fact that $p_n$ is
  non-decreasing.
\end{proof}
\section{Properties of RGGs in general domains via tessellation}
\label{sec:tessellation}

In the previous section we considered four models of RGGs defined
over the \emph{convex} domain~$[0,1]^d$. We discussed the necessary
conditions such that random graphs will be localizable for any
monotone property $\A$. In this section we consider the specific
monotone properties of \emph{connectivity} and \emph{bounded stretch}
for general domains.

A region $\Gamma \subset [0,1]^d$ is said to be \emph{$\rho$-safe} for
some $\rho >0$ if $\B_\rho(\Gamma) \subset [0,1]^d$, namely if the
Minkowski sum of $\Gamma$ with a ball of radius~$\rho$ is contained in
$[0,1]^d$.

\subsection{Connectivity}
Denote by $\aconn$ the connectivity property.  We show that for any
random graph $\G_n$ which is an RGG, RBG, SRGG or REGG that is
localizable for $\aconn$ it also holds that $\G_n$ is connected over
any $\rho$-safe region $\Gamma\subset [0,1]^d$. Note that we make no
additional assumptions on $\Gamma$ in this section.

\begin{theorem}\label{thm:conn}
  Let $\Gamma \subset [0,1]^d$ be a $\rho$-safe region for some
  constant $\rho > 0$ independent of $n$ and let $\G_n$ be a random
  graph that is localizable for $\aconn$.  Then any two points
  $x,y \in \Gamma \cap \X_n$ that lie in the same connected component
  of $\Gamma$ are connected in $\G_n(\B_{\rho}(\Gamma))$ \as.
\end{theorem}
In the proof of Theorem~\ref{thm:conn} we will place two
partially-overlapping grids over~$\Gamma$ and use the localization of
$\G_n$ in each grid cell (see Fig.~\ref{fig:grids}).  We now proceed
to define the grids and state several of their properties which, in
turn, will allow us to formally prove Theorem~\ref{thm:conn}.

Let $\dH_{\eps}$ be a grid partition of $[0,1]^d$ into axis-aligned
hypercubes with side length of~$\eps=\frac{2}{3 \sqrt{d}} \rho$.
Furthermore, denote by $\dH_{\eps}(\Gamma)$ the subset of cells of
$\dH_{\eps}$ whose intersection with~$\Gamma$ is non-empty.  Namely,
$\dH_{\eps}(\Gamma) = \{ H \in \dH_{\eps}\ | \ H \cap \Gamma \neq
\emptyset\}$.
Let~$\tilde{\dH}_{\eps}$ be a grid partition of $[0,1]^d$ into
axis-aligned hypercubes with side length of $\eps$ obtained by
shifting $\dH_{\eps}$ by $\eps/2$ along every axis and let
$\tilde{\dH}_{\eps}(\Gamma) = \{ H \in \tilde{\dH}_{\eps}\ | \ H \cap
\dH_{\eps}(\Gamma) \neq \emptyset\}$.  We have the following claim.

\begin{claim}
  \label{clm:collision_free}
  Let
  $H\in \dH_{\eps}(\Gamma) \cup
  \tilde{\dH}_\eps(\Gamma)$. Then $H \subset \B_\rho(\Gamma)$.
\end{claim}

\begin{proof}
  Consider a hypercube $H\in \dH_{\eps}(\Gamma)$.  By the definition
  of $\dB_{\eps}(\Gamma)$, $H$ intersects $\Gamma$ and let
  $x \in H \cap \Gamma$ be some intersection point.  Since
  $x \in \Gamma$, we have that $ \|x-y\|_2 > \rho$ for any
  point~$y \notin \B_\rho(\Gamma)$.  Recall that $H$ is an
  axis-aligned hypercube with side length of
  $\eps = \frac{2}{3\sqrt{d}}\rho$.  Thus, the maximal distance
  between any two points in $H$ is $\eps \sqrt{d} = \frac{2}{3}\rho$.
  Using the triangle inequality we have for every point $x' \in H$ and
  for any point~$y \notin \B_\rho(\Gamma)$,
  \[
  \|x'-y\|_2 \geq \|x-y\|_2 - \|x-x'\|_2 > \rho - \frac{2}{3}\rho =
  \frac{1}{3}\rho > 0,
  \]
  which implies that $x'\in \B_\rho(\Gamma)$.

  The proof for a hypercube $\tilde{H}\in \tilde{\dH}_{\eps}(\Gamma)$
  follows similar lines using the fact that for any
  point~$x' \in \tilde{H}$ and any point~$x \in H$ such that
  $H \cap \tilde{H} \neq \emptyset$ we have that
  $ \|x-x'\|_2 \leq \frac{3}{2}\eps \sqrt{d} = \rho$.
\end{proof}

We introduce some more terminology.  Every two
cells~$H,H'\in \dH_{\eps}(\Gamma)$ are called \emph{neighbors} if they
share a $(d-1)$-dimensional face. We now consider a refinement of
each grid cell~$H$ of~$\dH_\eps(\Gamma)$ (or of
$\tilde{\dH}_\eps(\Gamma)$) into $2^d$ sub-cells obtained by
splitting~$H$ by two along each axis through the middle point of $H$.
This induces the set of (refined) grid cells $\dH_{\eps / 2}(\Gamma)$
(or $\tilde{\dH}_{\eps / 2}(\Gamma)$, respectively).  Note that the
number of cells in~$\dH_{\eps / 2}(\Gamma)$ and
$\tilde{\dH}_{\eps / 2}(\Gamma)$ is fixed for the given
$d,\rho,\Gamma$, and does not depend on~$n$.

\begin{claim}\label{clm:nonempty}
  Let $H\in \dH_\eps(\Gamma) \cup \tilde{\dH}_\eps(\Gamma)$ (similarly
  for $H\in \dH_{\eps/2}(\Gamma) \cup \tilde{\dH}_{\eps/2}(\Gamma)$) . Then
  $\X_n\cap H\neq \emptyset$, \as.
\end{claim}

\begin{proof}
  We show the proof for $\dH_{\eps}(\Gamma)$, and the proof for
  $\tilde{\dH}_{\eps}(\Gamma)$ follows similar lines.  We start by
  showing that the probability that a specific
  cell~$H\in \dH_{\eps}(\Gamma)$ does not contain a sample of~$\X_n$
  tends to $0$:
  \[\Pr[\X_n\cap H=\emptyset]=(1-|H|)^n
  = (1-\eps^d)^n \leq e^{-n\eps^d}.\]
  Using the union bound, we deduce,
  \begin{align*}
    \Pr[\exists H\in \dH_\eps(\Gamma):\X_n\cap H= \emptyset]\leq
    \sum_{H\in \dH_\eps(\Gamma)}\Pr[\X_n\cap H=\emptyset]\leq
    be^{-n\eps^d},
  \end{align*}
  where $b$ denotes the number of cells in $\dB_\eps(\Gamma)$. As $b$
  is independent of $n$ the last expression tends to $0$ as $n$ tends
  to $\infty$.
\end{proof}

We are ready for the main proof.
\begin{proof}[Proof (Theorem~\ref{thm:conn})] 
  Recall that $\G_n$ is localizable for~$\aconn$.  As
  $\bigcup_{H \in \dH_{\eps}(\Gamma)} \subset \B_{\rho}(\Gamma)$, and since
  $x$ and $y$ are in the same connected component of~$\Gamma$,
  there exists a sequence of hypercubes $H_1, \ldots ,H_k\in\dH_{\eps}(\Gamma)$ such that
  (i)~$x \in H_1$, (ii)~$y \in H_k$ and (iii)~$H_{i}$ and $H_{i+1}$
  are neighbors for $1 \leq i < k$.  By Claim~\ref{clm:collision_free}
  each $H_i$ is  contained in~$\B_{\rho}(\Gamma)$.
  
  Claim~\ref{clm:nonempty} ensures, using the fact that $\Gamma$ is $\rho$-safe, that each~$H_i$
  contains a vertex of $\G_n$ \as  Let
  $x = x_1, \ldots, x_k = y$ denote such a set of vertices where
  $x_i \in H_i$.  We will show (using the localization of monotone
  properties) that $x_i$ and $x_{i+1}$ are connected in~$\G_n(\B_{\rho}(\Gamma))$
  which will conclude our proof.

  Let $\tilde{H} \in \tilde{\dH}_{\eps}(\Gamma)$ be a hypercube
  that intersects both~$H_i$ and~$H_{i+1}$ (there are always $2^{d-1}$
  such hypercubes).  By Claim~\ref{clm:nonempty},
  both~$\tilde{H} \cap H_i$ and ~$\tilde{H} \cap H_{i+1}$ contain a
  vertex of~$\G_n$ \as, since both of these intersection represent
  hypercubes in~$\dH_{\eps/2}(\Gamma)$. Let~$z_i$ and $z_{i+1}$ be these
  vertices, respectively (see Fig.~\ref{fig:connection}).

  Using Lemmas~\ref{lem:local_rgg}-\ref{lem:local_regg} we have that $x_i$ and $z_i$
  are connected in $H_i$, that $z_i$ and~$z_{i+1}$ are connected in
  $\tilde{H}$, and that $z_{i+1}$ and $x_{i+1}$ are connected in $H_{i+1}$
  \as This must hold for every $1\leq i< k$ in
  order to ensure that $x$ and $y$ are connected in
  $\G_n(\B_{\rho}(\Gamma))$. Due to the fact that $k$ can be at most the number
  of cells in $\dH_\eps(\Gamma_\rho)$, which is independent of~$n$, we
  deduce that indeed $x,y$ are connected in $\G_n(\B_{\rho}(\Gamma))$ \as
\end{proof}

\begin{figure}[tb]
\centering
\includegraphics[height =5.5 cm]{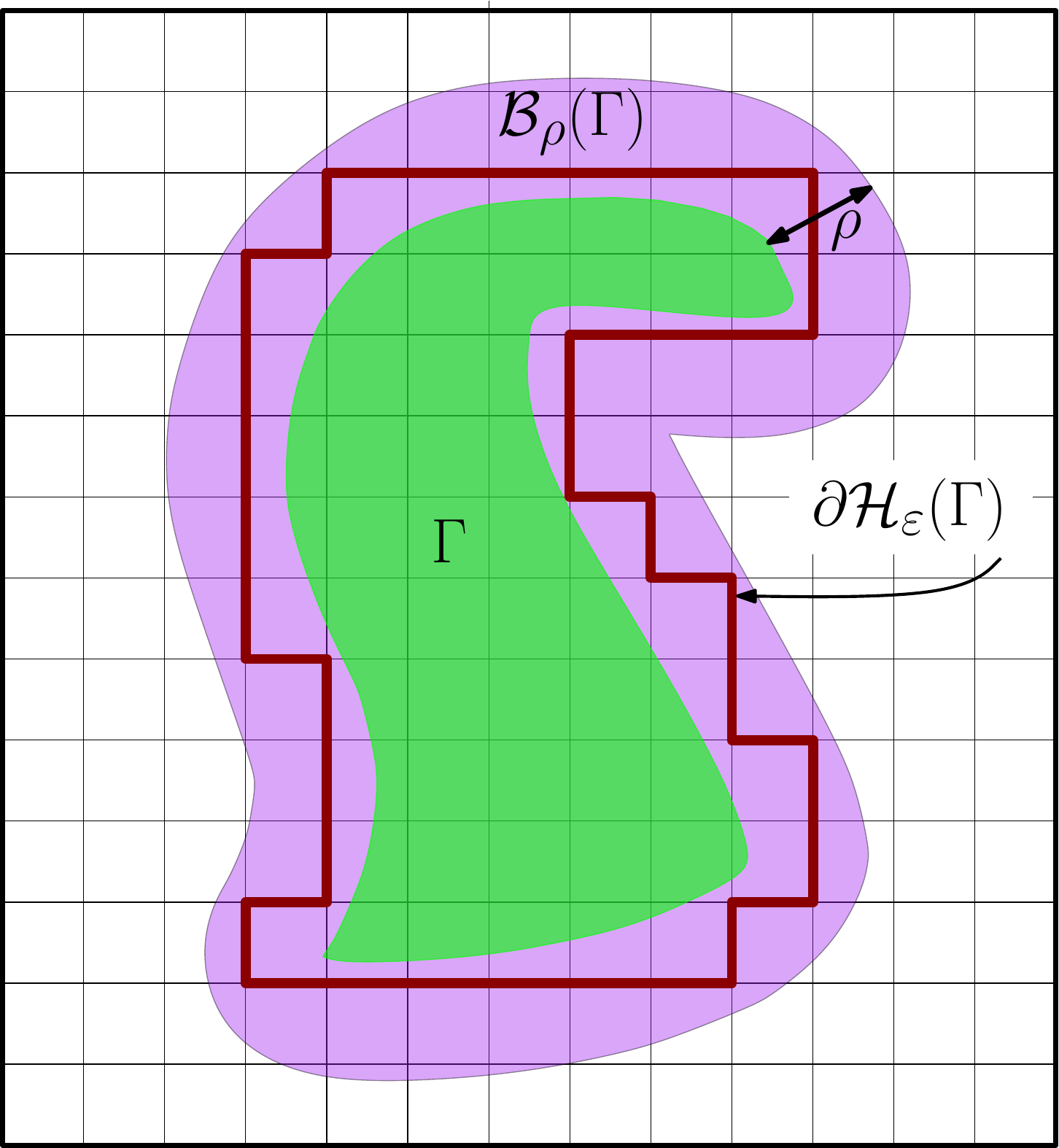}
  \caption{\sf \footnotesize Visualization of $\Gamma$ (green),
    $\B_\rho(\Gamma)$ (purple) and the grid $\dH_{\eps}$ used for the proof
    of Theorem~\ref{thm:conn}.  The boundary of the set of grid cells
    $\dH_{\eps}(\Gamma)$ is depicted using dark red lines.}
  \label{fig:grids}
\end{figure}

\begin{figure}[tb]
  \centering
  \includegraphics[height =5.5 cm]{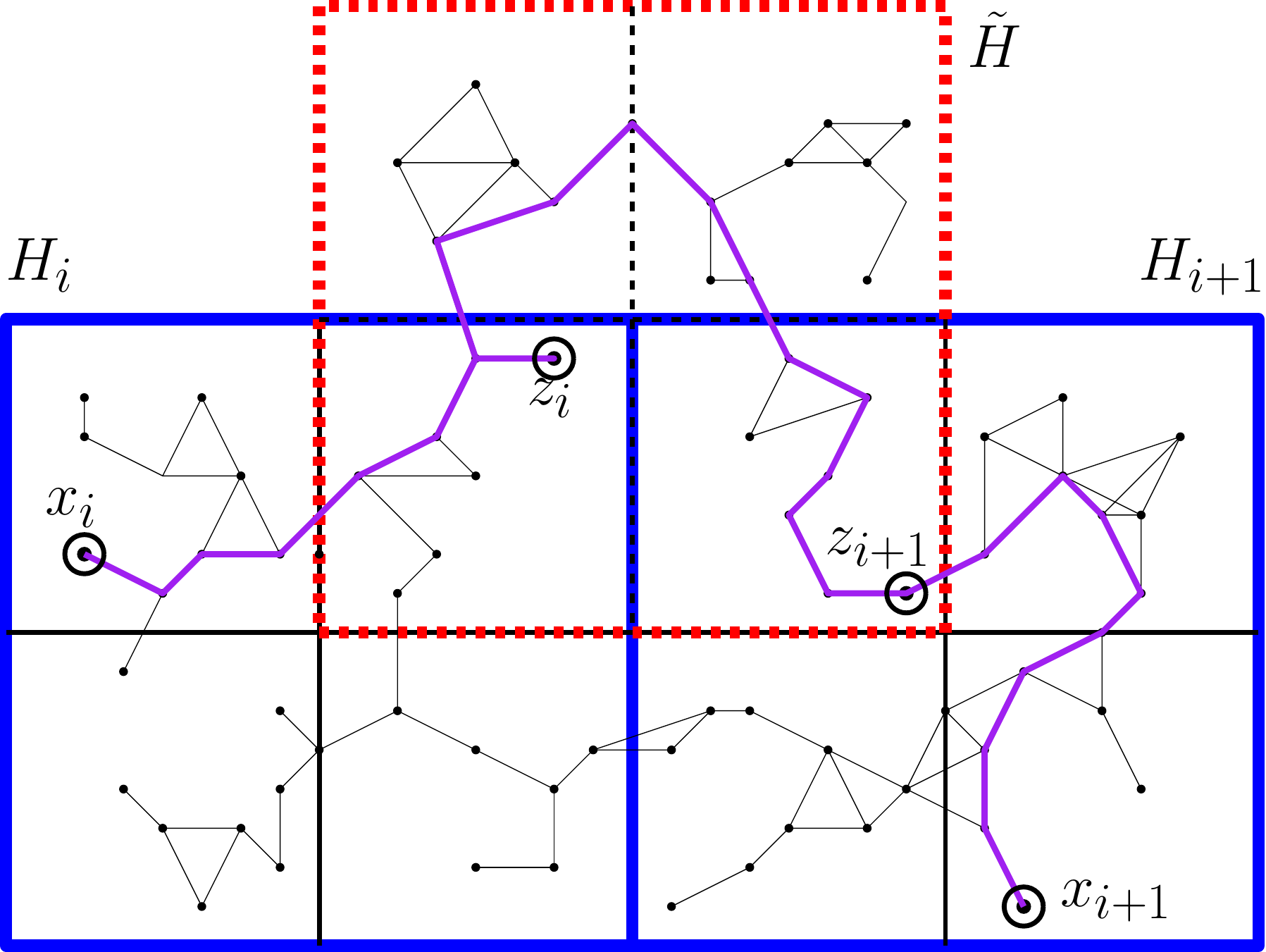}
  \caption{\sf \footnotesize Visualization of the proof of
    Theorem~\ref{thm:conn}.  Hypercubes $H_i,H_{i+1}$ and $\tilde{H}$
    of side length $\eps$ are depicted in solid blue lines and dashed
    red lines, respectively.  A path connecting $x_i \in H_i$ to
    $x_{i+1} \in H_{i+1}$ via intermediate points
    $z_i \in H_i \cap \tilde{H}$ and
    $z_{i+1} \in H_{i+1} \cap \tilde{H}$ is depicted by a purple
    line. }
	\label{fig:connection}
\end{figure}

\subsection{Bounded stretch}
Given $\zeta \geq 1 $ denote by $\abd^\zeta$ the property indicating
that a given geometrically-embedded graph has a \emph{bounded stretch}
of~$\zeta$, for any two vertices. Formally, let $\G$ be a graph
defined over a vertex set $X\subset[0,1]^d$. The notation
$\G\in \abd^\zeta$ indicates that for every $x,y\in X$ it holds that
$\dist(\G,x,y)\leq \zeta\|x-y\|_2$.  The proof of the following theorem
 is very similar to that of Theorem~\ref{thm:conn}.

\begin{theorem}\label{thm:bd}
  Let $\Gamma \subset [0,1]^d$ be a $\rho$-safe region for some
  constant $\rho > 0$ independent of $n$.  Let $\G_n$ be a random
  graph that is localizable for $\abd^\zeta$, for
  some~$\zeta \geq 1 $.  Additionally, let $x,y\in \X_n$ be two points
  that lie in the same connected component of~$\Gamma$. Then
  $\dist(\G_n(\B_{\rho}(\Gamma)),x,y)\leq \zeta \|x-y\|_{\Gamma}
  +o(1)$
  \as, where $\|x-y\|_{\Gamma}$ denotes the length of the shortest
  path between $x$ and $y$ that is fully contained in $\Gamma$.
\end{theorem}

\begin{figure*}[tb]
\centering
\begin{minipage}{.45\textwidth}
  \centering
  \includegraphics[height =5.5 cm]{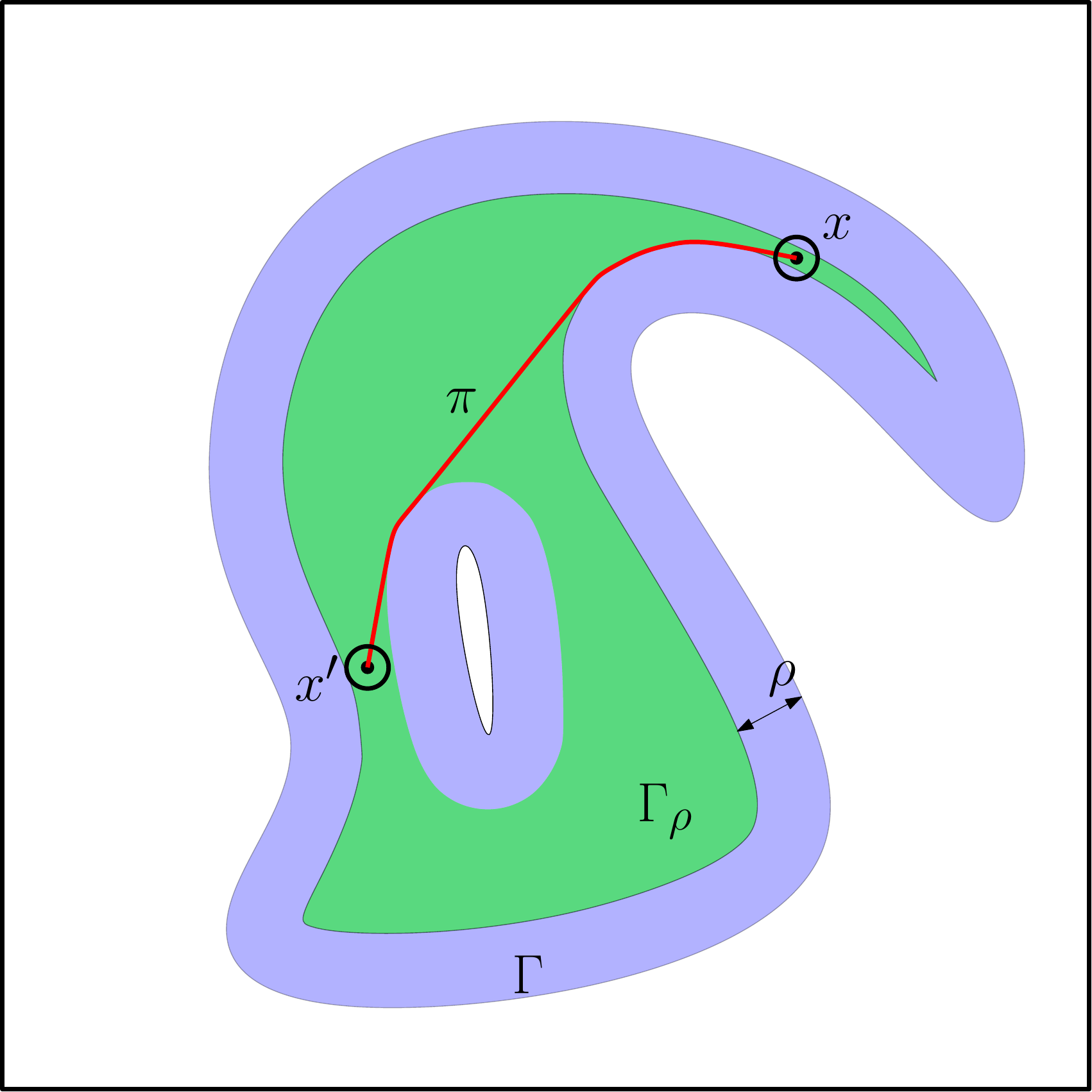}
  \caption{\sf \footnotesize Visualization of shortest path $\sigma$
    (red) connecting two points $x, y$ within $\Gamma$ (green)
    used for the proof of Theorem~\ref{thm:bd}. }
    \vspace{1.15 cm}
     \label{fig:shortest_path}
\end{minipage}%
\hfill
\begin{minipage}{.45\textwidth}
  \centering
  \includegraphics[height =5.5 cm]{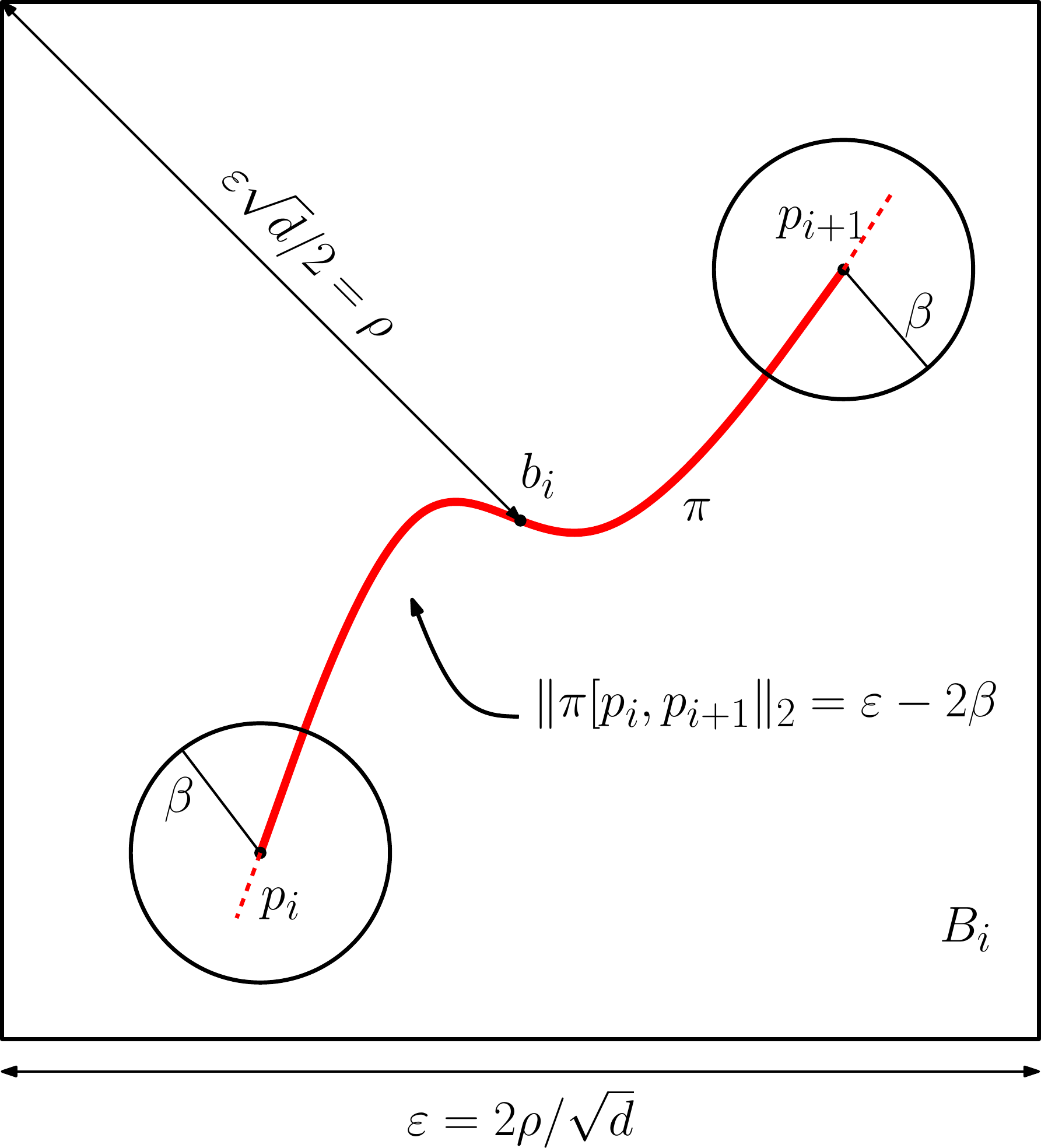}
  \caption{\sf \footnotesize Visualization of proof of
    Theorem~\ref{thm:bd}.  Points $p_i$ and $p_{i+1}$ along path $p_i$
    are connected such that the distance between the points along
    $\sigma$ is exactly $\eps/2$.  Hypercube~$H_i$ of side length
    $\eps$ centred at $q_i$ (midpoint between $p_i$ and
    $p_{i+1}$).  }
  \label{fig:construction}
 \end{minipage}
\end{figure*}

\begin{proof}
  Let $\sigma$ be the shortest path connecting $x$ and $y$, which is
  entirely contained in $\Gamma$ (see Fig.~\ref{fig:shortest_path}). We provide a brief summary of the
  proof. We define a sequence of $b$ points $p_1,\dots,p_b$ that are
  equally spaced along $\sigma$ (here $b$ plays a similar role to the number of hypercubes used in the proof of Theorem~\ref{thm:conn}). Next, we show that for every $p_i$
  there is a vertex $x_i$ of $\X_n$ that is sufficiently close to~$p_i$. 
  Moreover, we show that for every $1\leq i < b$, the
  points $x_i,x_{i+1}$ are contained in a hypercube $H_i$ whose size
  is independent of $n$, and which is contained in
  $\B_{\rho}(\Gamma)$.  This allows to exploit the fact that $\G_n$
  is localizable for $\abd^\zeta$ and show that $\G_n$ contains a path
  from $x_i$ to $x_{i+1}$ that is similar in length to the subpath of
  $\sigma$ connecting $p_i$ to $p_{i+1}$.

  Set $R=R_n:=\theta_d^{-1/d}\left(\frac{\log n}{n}\right)^{1/d}$ and
  $\eps = 2\rho / \sqrt{d}$.  Note that for sufficiently large $n$ it
  follows that $R < 2\rho$.  Consider the sequence of $b$ points
  $P=p_1, \ldots, p_b$ in~$[0,1]^d$ along $\sigma$ such that
  (i)~$p_1 = x$, (ii)~$p_b = y$ and (iii)~the subpath of $\sigma$
  between points $p_i$ and $p_{i+1}$ has length exactly $\eps/2$
  (except for, possibly, the last subpath). Note that
  $b=2|\sigma|/\eps$ is finite and independent of $n$.
  See Fig.~\ref{fig:construction}.

  For every $p_i\in P$ define $x_i=\argmin_{x\in \X_n}\|x-p_i\|_2$,
  namely the closest point from $\X_n$ to $p_i$. We show that \as for
  every $1\leq i\leq b$ it holds that $\|x_i-p_i\|_2\leq R_n$.
  Similarly to the proof of Claim~\ref{clm:nonempty}, for a given
  $1\leq i\leq b$ we have that
  \[\Pr[\X_n\cap
  \B_{R_n}(x_i)=\emptyset]=(1-|\B_{R_n}(x_i)|)^n=(1-\theta_dR_n^d)^n\leq
  e^{-\log n} \leq 1/n.\] Then, we use the union bound to establish the bound
  \[\Pr[\exists 1\leq i\leq b, \X_n\cap \B_{R_n}(x_i)=\emptyset] \leq
  b / n, \] which converges to $0$ as $n\rightarrow \infty$.

  Now, set~$q_i$ to be the point midway between~$p_i$ and $p_{i+1}$ on
  $\sigma$. Additionally, define $H_i$ to be the axis-aligned
  hypercube of side-length $\eps$ centred at $q_i$.  Note that for
  sufficiently large $n$ it holds that $x_i,x_{i+1}\in H_i$ \as
  Moreover, it can be shown using an argument similar to the one used
  in Claim~\ref{clm:collision_free} that $H_i\subset\B_\rho(\Gamma)$.
  This, combined with the fact that $\G_n$ is localizable for
  $\abd^\zeta$, yields that the following holds \as:
  \begin{align*}
    \dist(\G_n(\H_i),x_i,x_{i+1}) &\leq \zeta\|x_i-x_{i+1}\|_2 +o(1) \\
                                  & \leq
                                    \zeta(\|p_i-p_{i+1}\|_2 + 2R_n)
                                    +o(1)\\ &
                                              =\zeta\|p_i-p_{i+1}\|_2+o(1) \\
                                  & = \zeta\|p_i-p_{i+1}\|_\Gamma + o(1).
  \end{align*}
  
  Recall that $b$ is finite and independent of $n$. Thus, the
  following inequality, which concludes the proofs, holds \as:
  \begin{align*}
    \dist(\G_n(\B_\rho(\Gamma),x,y) &\leq \sum_{i=1}^{b-1}\dist(\G_n(H_i),x_i,x_{i+1})
    \\ & \leq  \sum_{i=1}^{b-1}  \zeta\|p_i-p_{i+1}\|_\Gamma + o(1) \\
                                    & \leq \zeta\|x-y\|_\Gamma+o(1).                   
  \end{align*}
\end{proof}

\begin{remark}
  The proof of Theorem~\ref{thm:conn} could be altered to use the same
  arguments presented for the proof of Theorem ~\ref{thm:bd}, i.e.,
  follow a specific path instead of constructing a grid over the
  entire $\Gamma$.
\end{remark}

\section{Application to sampling-based motion planning}
\label{sec:mp}
We now move to the setting of motion planning in which a robot
operates in the configuration space $\C = [0,1]^d$, and whose free space is
denoted by $\F\subseteq \C$. 
Recall that the problem consists of
finding a continuous path between two configurations (points)
$s,t\in \F$, that is fully contained in $\F$.

The reason why we cannot apply results on RGGs to motion
planning directly is that $\F$ is not the full
hypercube $[0,1]^d$ but rather could be a geometrically and
topologically very complicated subset of this hypercube.
However, the localization-tessellation approach that we
have devised enables us to fairly directly adapt results
from the theory of RGGs to this more involved setting, as
we do in this section.

Specifically, we start by introducing the Bluetooth-PRM, Soft-PRM and
Embedded-PRM algorithms, which are extensions of RBG, SRGG and REGG to
the setting of motion planning. We then continue to provide proofs for
probabilistic completeness and asymptotic optimality of these methods.

\begin{remark}
  Soft-PRM and Bluetooth-PRM are very similar to a technique that was
  studied experimentally by McMahon et al.~\cite{McmETAL12}; here we
  provide theoretical analysis for it.
\end{remark}

\subsection{Motion-planning algorithms}
We introduce the Soft-PRM algorithm. The description of PRM and
Embedded-PRM immediately follow, as they are special cases of
Soft-PRM. Recall that SRGG is defined for a connection radius $r_n$
and the function $\phi_n:\dR^+\rightarrow [0,1]$: two vertices
$x,y\in \X_n$ for which $\|x-y\|_2\leq r_n$ are connected with an edge
with probability $\phi_n(\|x-y\|_2)$.

We use the following standard procedures: 
\texttt{sample}$(n)$
returns $n$ configurations that are sampled uniformly and randomly
from $\C$; \texttt{nearest\_neighbors}$(x,V,r)$ returns all the
configurations from $V$ that are found within a distance of $r$ from
$x$; \texttt{collision\_free}$(x,y)$ tests whether the straight-line
segment connecting $x$ and $y$ is contained in $\F$;
\texttt{random\_variable}$()$ selects uniformly at random a real
number in the range $[0,1]$.

The \emph{preprocessing phase} of Soft-PRM is described in
Alg.~\ref{alg:soft}. In lines~1-4, $n$ configurations are sampled
(note that this slightly differs from some PRM descriptions in which
the samples are assumed to be collision free) and for each sample,
Soft-PRM retrieves the neighboring samples which are within a distance
of at most $r_n$ from it. For each sample point~$x$ and each candidate
neighbor~$y$ it decides with probability $\phi_n(\|x-y\|_2)$ whether
to attempt the connection (lines~5-6).  If this is the case, the
edge~$(x,y)$ is tested for being collision free (line 7), and added
accordingly to~$E$.

The (standard) PRM and Embedded-PRM are identical to
Alg.~\ref{alg:soft} using the parameters $r_n$ and $\phi_n = 1$ for
PRM and $r_n = \infty$ and $\phi_n = p_n$ for Embedded-PRM.  Note that
in the implementation of Embedded-PRM there is no need to maintain a
nearest-neighbor data structure (line 3) as every pair of vertices
$x,y\in\X_n$ is chosen with probability $p_n$. Bluetooth-PRM can be
obtained by replacing lines 5,6 in Alg.~\ref{alg:soft} with a suitable
procedure which uniformly samples a subset of $c_n$ neighbors from a
given collection.

In the \emph{query stage} each of the three algorithms is given two
configurations $s,t$, which are then connected to their neighbors in
the underlying roadmap by executing \texttt{nearest\_neighbors} with
the connection radius
$r_{\text{query}} = \gamma\left(\frac{\log n}{n}\right)^{1/d}$, where
$\gamma>\gamma^*= 2(2d\theta_d)^{-1/d}$.
Naturally, every connection is tested for collision. Finally, the
underlying graph is searched for the shortest path from $s$ to $t$ and
the respective path in $\F$ is returned (if exists).

\begin{observation}
  Denote by $\G_n$ the \textup{Soft-PRM} roadmap produced for $n$
  samples and the connection radius $r_n$. Then
  $\G_n=\Gsoft_n(\X_n;r_n;\phi_n)\cap \F$. The same applies for the
  relation of the underlying roadmaps of
  \textup{PRM},\textup{Bluetooth-PRM}, \textup{Embedded-PRM}, and RGG,
  RBG, REGG, respectively.
\end{observation}
\begin{algorithm}[tb]
  \caption{Soft-PRM($n, r_n,\phi_n$)}
  \label{alg:soft}
  \begin{algorithmic}[1]
	\State
      $V \gets \{ x_{\text{init}} \} \cup \texttt{sample}(n)$;
      $E \gets \emptyset$; $\G \gets (V,E)$ 
      \ForAll{$x\in V$} 
      \State $U\gets\texttt{nearest\_neighbors}(x,V,r_n)$
      \ForAll {$y \in U$}
      \State $\xi\gets\texttt{random\_variable}()$
      \If {$\xi \leq \phi_n(\|x-y\|_2)$}
      \If {\texttt{collision\_free}($x, y$)}
      \State	$E \leftarrow E \cup (x,y)$
      \EndIf
      \EndIf
      \EndFor
      \EndFor
      \Return $\G$
  \end{algorithmic}
\end{algorithm}
\subsection{Probabilistic completeness}
Let $(\F, s, t)$ be a motion-planning problem that consists of the
free space $\F \subset [0,1]^d$, and $s,t\in \F$ are the start and
target configurations, respectively. We provide the definition of
\emph{probabilistic completeness} and state the conditions under which
the aforementioned algorithms posses this property. 
\begin{definition}[\cite{KF11}]
  Let $\sigma:[0,1]\rightarrow \F$ be a continuous path, and let
  $\delta>0$. The path $\sigma$ is $\delta$-robust if
  $\B_\delta(\sigma) \subseteq \F$.
\end{definition}
\begin{definition}[\cite{KF11}]
  A motion-planning problem $(\F, s, t)$ is \emph{robustly
    feasible} if there exists a $\delta$-robust path $\sigma$
  connecting~$s$ to $t$, for some fixed $\delta >0$.
\end{definition}
\begin{definition}[\cite{KF11}]
  A planner ALG is probabilistically complete if for any
  robustly-feasible $(\F, s, t)$, the probability that ALG finds a
  solution with $n$ samples converges to $1$ as $n$ tends to $\infty$.
\end{definition}
\begin{lemma}\label{lem:complete_planners}
  Let
  $\textup{ALG}\in \{\textup{PRM}, \textup{Bluetooth-PRM},
  \textup{Soft-PRM}, \textup{Embedded-PRM}\}$
  with a selection of parameters for which the corresponding random
  graph $\G_n$ be \emph{localizable for connectivity}. Then
  \textup{ALG} is probabilistically complete.
\end{lemma}
\begin{proof} 
  Suppose that $(\F,s,t)$ is robustly feasible.  By definition, there
  exists a path $\sigma$ connecting $s$ to $t$ and $\delta>0$ for
  which $\B_{\delta}(\sigma)\subseteq \F$.

  Let $\delta':=\delta/2$ and let $\Gamma:={\B}_{\delta'}(\sigma)$.
  Note that $s,t\in \Gamma$ and also
  ${\B}_{\delta'}(\Gamma) = {\B}_{\delta}(\sigma) \subseteq {\F}
  \subset [0,1]^d$,
  which implies that $\Gamma$ is $\delta'$-safe.  By the fact that ALG
  is localizable for connectivity, and by Theorem~\ref{thm:conn}, we
  have that for every $x,y\in \X_n\cap \Gamma$ it follows that $x,y$
  are connected in $\G_n(\B_{\delta'}(\Gamma))$ \as

  It remains to show that during the query stage~$s$ and~$t$ are
  connected to $\G_n(\B_{\delta'}(\Gamma))$. It is not hard to verify
  that
  $\X_n\cap \B_{r_{\text{query}}}(s)\neq \emptyset, \X_n\cap
  \B_{r_{\text{query}}}(t)\neq \emptyset$,
  \as, which implies connectivity.
\end{proof}

\begin{theorem}\label{thm:complete_final}
  Recall that $\gamma^* = 2(2d\theta_d)^{-1/d}, \gamma^{**}=d\cdot 2^{1+1/d}$ and that $d\geq 2$.
  Then the following algorithms are probabilistically complete:
  \begin{description}
  \item[(i)] \textup{PRM($r_n$)}, where
    $r_n=\gamma\left(\frac{\log n}{n}\right)^{1/d}$, for
    $\gamma>\gamma^*$;
  \item[(ii)] \textup{Bluetooth-PRM($r_n;c_n$)}, where $r_n=\rtrs$, for
    $\gamma>\gamma^{**}$ and $c_n>\sqrt{\frac{2\log n}{\log\log n}}$;
  \item[(iii)] \textup{Soft-PRM($r_n;\phi_n$)}, where $r_n=\rtrs$, for
    $\gamma>(d+1)^{1/d}\gamma^*$ and $\phi_n(z)=1-z/r_n$, for any
    $z\in \dR^+$;
  \item[(iv)] \textup{Embedded-PRM($p_n$)}, where
    $p_n = \omega\left(\frac{\log^d n}{n}\right)$.
  \end{description}
\end{theorem} 

Item~\textbf{(i)} follows from combining Theorem~\ref{thm:rgg_con}
with the localization lemma for RGGs (i.e.,
Lemma~\ref{lem:local_rgg}), and Lemma~\ref{lem:complete_planners}. The
other items similarly follow.

\begin{remark}
  The conditions in \textbf{(i)} are not only sufficient but also
  necessary, according to Theorem~\ref{thm:rgg_con}.
\end{remark}

\begin{remark}
  The connection radius in~\textbf{(i)} is smaller by a factor of
  $2^{-1/d}$ than the one obtained by Janson et al.~\cite{JSCP15}, and
  smaller by a factor of $2^{-1/d}(d+1)^{-1/d}$ than the connection
  radius proposed by Karaman and Frazzoli~\cite{KF11} when no
  obstacles are present.
  We also mention
  that, similarly to these two works, $r_n$ can be reduced by a factor
  of $|\F|^{1/d}$, with a slight modification to
  Theorem~\ref{thm:conn}.
\end{remark}

\subsection{Asymptotic (near-)optimality}
Given a path $\sigma$ denote its length by $|\sigma|$. We define the
property of asymptotic near-optimality and state the conditions under
which PRM and Embedded-PRM have this property.
\begin{definition}
  Suppose that $(\F, s, t)$ is robustly feasible. A path $\sigma^*$
  connecting $s$ to $t$ is \emph{robustly optimal} if it is a shortest
  path for which the following holds: for any $\eps > 0$ there exists a
  $\delta$-robust path~$\sigma$ such that
  $|\sigma|\leq (1+\eps)|\sigma^*|$ for some fixed~$\delta>0$.
\end{definition}
\begin{definition}
  A sampling-based planner ALG is \emph{asymptotically
    $\zeta$-optimal}, for a given $\zeta \geq 1$, if for every
  robustly-feasible problem $(\F,s,t)$ it follows that
  $|\sigma_n|\leq \zeta |\sigma^*|+o(1)$ \as, where $\sigma_n$ denotes
  the solution returned by ALG with $n$ samples. A planner that is
  asymptotically $1$-optimal is simply called \emph{asymptotically
    optimal}.
\end{definition}

\begin{lemma}\label{lem:complete_planners}
  Let
  $\textup{ALG}\in \{\textup{PRM}, \textup{Bluetooth-PRM},
  \textup{Soft-PRM}, \textup{Embedded-PRM}\}$
  with a selection of parameters for which the corresponding random
  graph $\G_n$ be localizable for $\abd^\zeta$, for $\zeta \geq 1$.
  Then \textup{ALG} is asymptotically $\zeta$-optimal.
\end{lemma}
\begin{proof}
  Suppose that $(\F,s,t)$ is robustly feasible. Denote by $\sigma^*$
  the robustly-optimal path. By definition, for every $\eps > 0$ there
  exists $\delta > 0$ and a $\delta$-robust path $\sigma_\eps$ such that
  $|\sigma_\eps|\leq (1+\eps)|\sigma^*|$. For now we will consider a
  fixed $\eps > 0$ and the corresponding path $\sigma_\eps$.

  Similarly to the proof of Lemma~\ref{lem:complete_planners}, Let
  $\delta':=\delta/2$ and let $\Gamma:={\B}_{\delta'}(\sigma_\eps)$. The
  query stage will succeed \as~and $s,t$ will be connected to some two
  vertices $x,y\in \X_n\cap \Gamma$ such that
  $\|s-x\|_2,\|t-y\|_2\leq R_n$. Observe that (\as)
  \[\|x-y\|_\Gamma\leq \|s-t\|_\Gamma+2R_n\leq |\sigma_\eps|+2R_n\leq
  (1+\eps)|\sigma^*|+o(1). \]
  Using this observation, together with Theorem~\ref{thm:bd}, and with
  the fact that $\G_n$ is localizable, we deduce that ALG finds a
  solution $\sigma_n$, which is contained in $\B_{\delta'}(\Gamma)$,
  such that $|\sigma_n|\leq \zeta(1+\eps)|\sigma^*|+o(1)$ \as

  We will now eliminate the $\eps$ factor from the distance bound.
  For a given $\eps >0$ and $n$, denote by $\P(n,\eps)$ the event
  \mbox{$\dist(\G_n(\Gamma),s,t) \leq \zeta(1+\eps)|\sigma^*| +o(1)$}.
  For every positive integer $i$ define $\eps_i:=1/i$. Let $n_i$ be
  the minimal integer such that for every $n\geq n_i$ it follows that
  $\Pr[\P(n,\eps_i)]\geq 1-\eps_i$. For a given $n$, let $i(n)$ be
  such that $n_i\leq n < n_{i+1}$. It follows that 
  \[\limn\Pr[\P(n,\eps_{i(n)})]\geq \limn 1-\eps_{i(n)}=1.\]
  As $\eps_{i(n)}=o(1)$ we may deduce that ALG is asymptotically
  $\zeta$-optimal.
\end{proof}
\begin{theorem}\label{thm:optimal_final}
  For $d\geq 2$ we have the following results:
  \begin{description}
  \item[(i)] \textup{PRM($r_n$)} is asymptotically $\zeta$-optimal for
    $r_n=\gamma\left(\frac{\log n}{n}\right)^{1/d}$, where
    $\gamma>\gamma^*$, and some constant $\zeta$;
  \item[(ii)] \textup{Embedded PRM($p_n$)} is asymptotically optimal
    for $p_n=\omega\left(\frac{\log^d n}{n}\right)$;
  \end{description}
\end{theorem}

\begin{remark}
  The conditions in \textbf{(i)} are not only sufficient but also
  necessary, as without them PRM will be incomplete.
\end{remark}
 
\section{Evaluation}\label{sec:evaluation}
In this section we present experiments demonstrating the behavior of
RGGs and SRGGs in the absence and in the presence of obstacles.  We
then proceed to compare the Soft-PRM and PRM algorithms.  For each
model, and each algorithm, we use the minimal parameters that are
required in order to ensure connectivity.  For our experiments we used
the Open Motion Planning Library (OMPL)~\cite{SMK12} with Randomly
Transformed Grids (RTG)~\cite{AKS14} as our nearest-neighbor (NN) data
structure.  RTG were shown to outperform other NN libraries for
several motion-planning algorithms~\cite{KSH15}.  All experiments were
run on a 2.8GHz Intel Core~i7 processor with 8GB of memory.  Results
are averaged over 100 runs, and computed for dimensions $d = 2, 6 , 9$
and~$12$.  Additionally, when results for different dimensions behave
similarly, we present only plots for $d=2$ and $d=12$.\vspace{20pt}

\noindent \textbf{Connectivity and stretch in the unit cube.}  We
begin by reporting the number of nodes that are \emph{not} in the
largest connected component (CC) for RGG and SRGG in the absence of
obstacles.  Clearly, when the graph is connected, this number is zero.
One can see (Fig.~\ref{fig:conn}) that as the number of nodes
increases, the number of nodes not in the largest CC approaches zero.
Additionally, the two models exhibit very similar trends.

We continue to asses how increasing the number of nodes affects the
stretch of the graphs. For each such $n$, we sampled $m = 50$ vertices
and computed the stretch for every pair of sampled vertices.  We then
report on the maximal stretch obtained among all $O(m^2)$ pairs of
nodes which gives a rough approximation of the average stretch of the
graph.  Results are depicted in Fig.~\ref{fig:stretch}.  Observe that
typically the stretch decreases as the number of nodes increases and
that RGG and SRRG behave very similarly. \vspace{5pt}

\noindent\textbf{Connectivity and stretch of RGGs in general domains.} 
The set of experiments come to demonstrate Theorems~\ref{thm:conn}
and~\ref{thm:bd}.  Namely, that the asymptotic behavior of RGGs with
respect to connectivity and stretch is maintained in the presence of
obstacles.  We constructed the following toy-scenario where we
subdivided the $d$-dimensional unit hypercube by halfing it along each
axis.  In the center of each one of the $2^d$ sub-cubes, we inserted
an axis-aligned hypercube as an obstacle.  The size of the obstacle
was chosen such that the obstacles covered $25\%$ of the unit
hypercube.  See Figure~\ref{fig:obs} for a visualization in two and
three dimensions.

\begin{figure}[t]
  \centering 
  	\subfloat [\sf ] 
  	{
	    \includegraphics[width=0.44\textwidth]{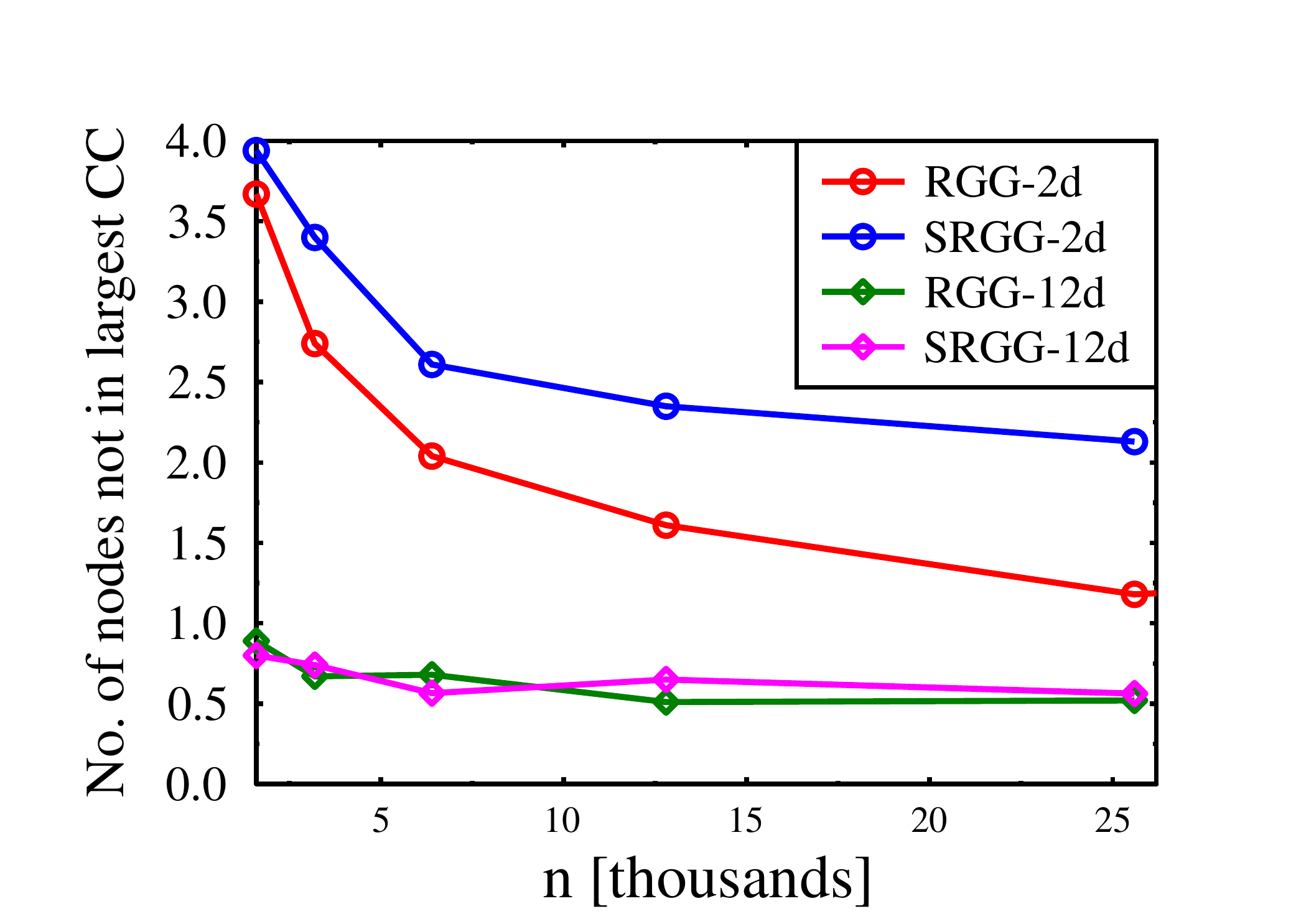}
		\label{fig:conn} 
	} 
	\subfloat [\sf ] 
	{
	    \includegraphics[width=0.44\textwidth]{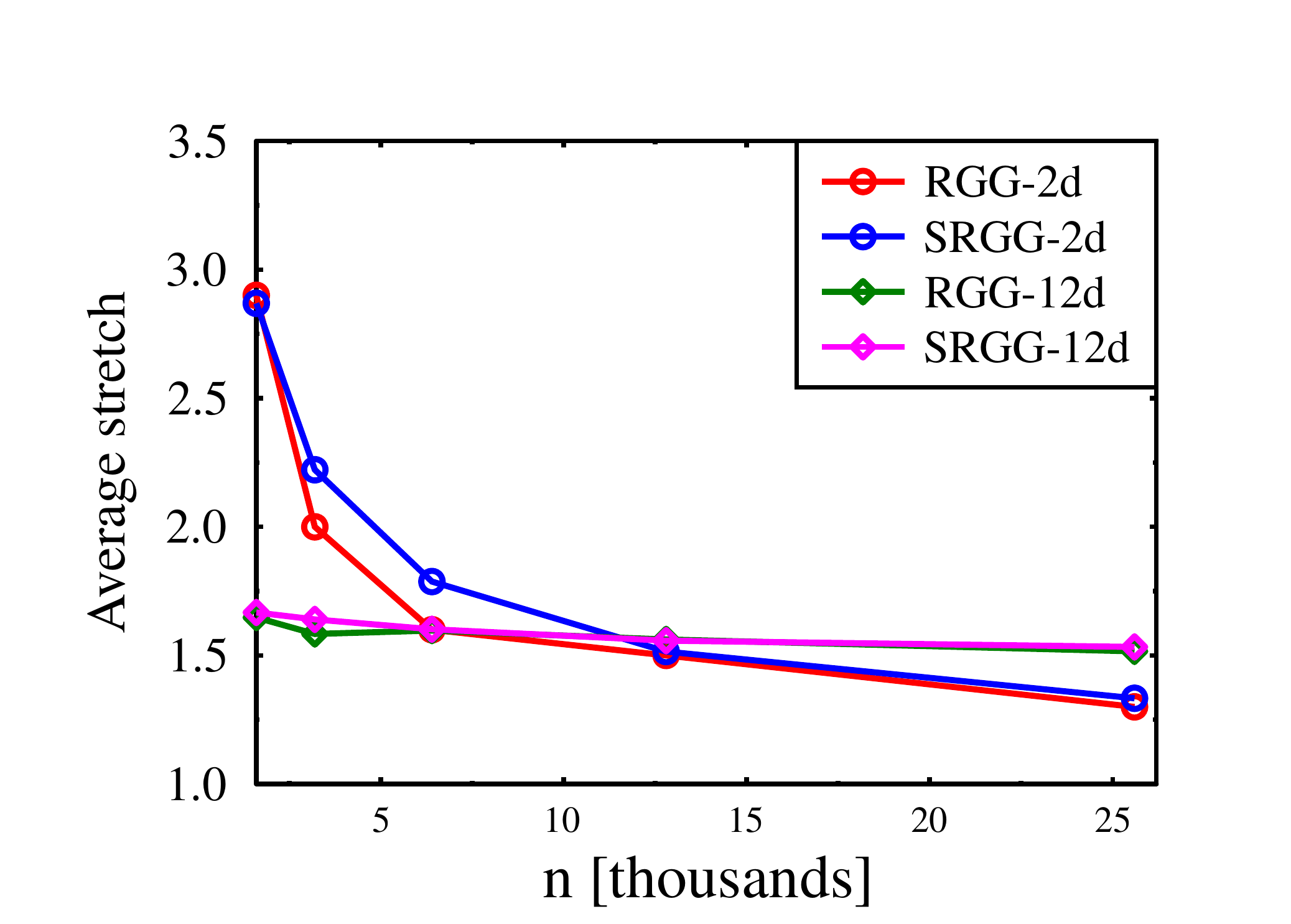}
		\label{fig:stretch} 
	}
	\caption{ \sf 	(a) Connectivity and (b) stretch in the absence of obstacles. 
} 
	\label{fig:no_obs}
\end{figure}
\begin{figure}[t]
  \centering 
  	\subfloat [\sf 2d] 
  	{
	    \includegraphics[width=0.4\textwidth]{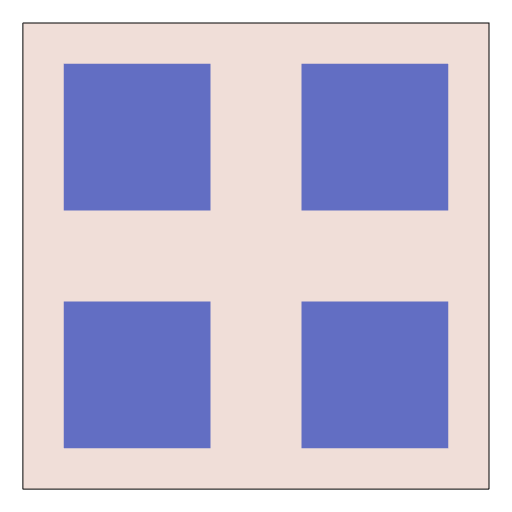}
		\label{fig:obs2D} 
	} 
	\hspace{5mm}
	\subfloat [\sf 3d] 
	{
	    \includegraphics[width=0.4\textwidth]{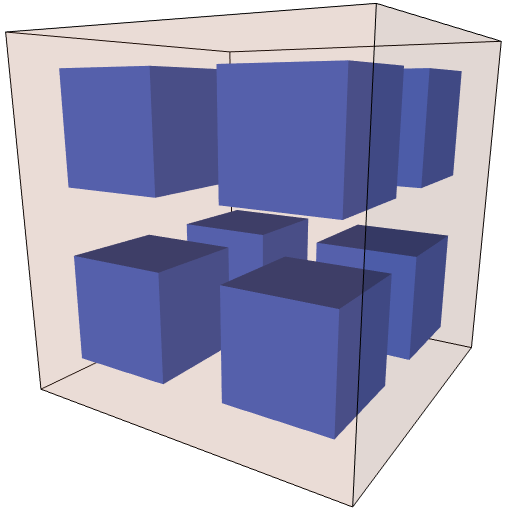}
		\label{fig:obs3D} 
	}
	\caption{ \sf Visualization of toy scenario for
				(a) two and (b) three
				dimensions.
				Obstacles are depicted in blue.} 
	\label{fig:obs}
\end{figure}

We report on the results for RGGs (Fig.~\ref{fig:res_obs}) and note that similar results were observed for SRGGs.
Stretch was computed between the origin $(0,\ldots,0)$ and the center
$(0.5,\ldots,0.5)$.
Observe that for all dimensions, the graph is asymptotically connected
and the stretch tends to one.

\begin{figure}[t]
  \centering 
  	\subfloat [\sf ] 
  	{
	    \includegraphics[width=0.48\textwidth]{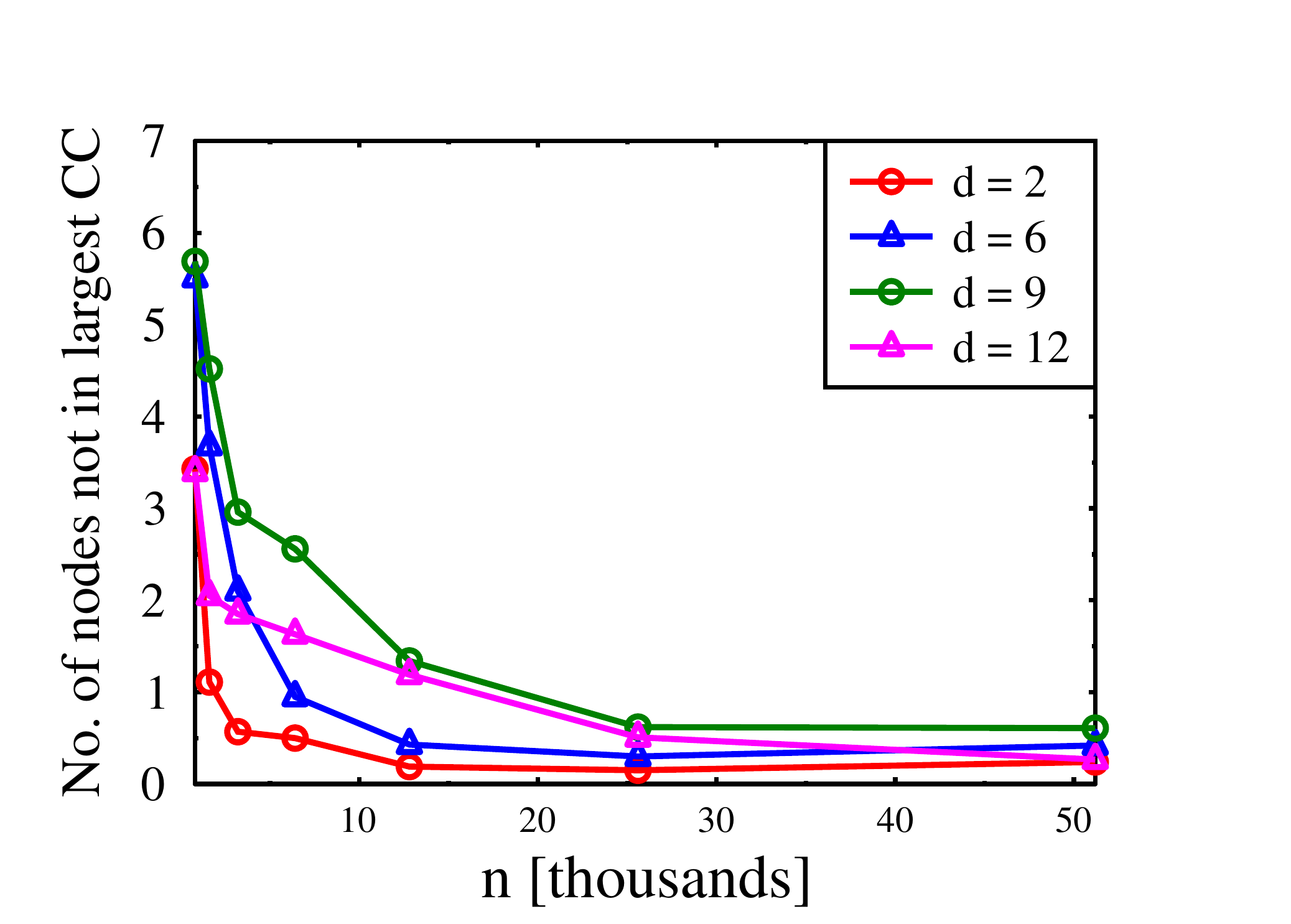}
		\label{fig:PRM_num_of_nodes_in_largest_cc_obs} 
	} 
	\subfloat [\sf ] 
	{
	    \includegraphics[width=0.48\textwidth]{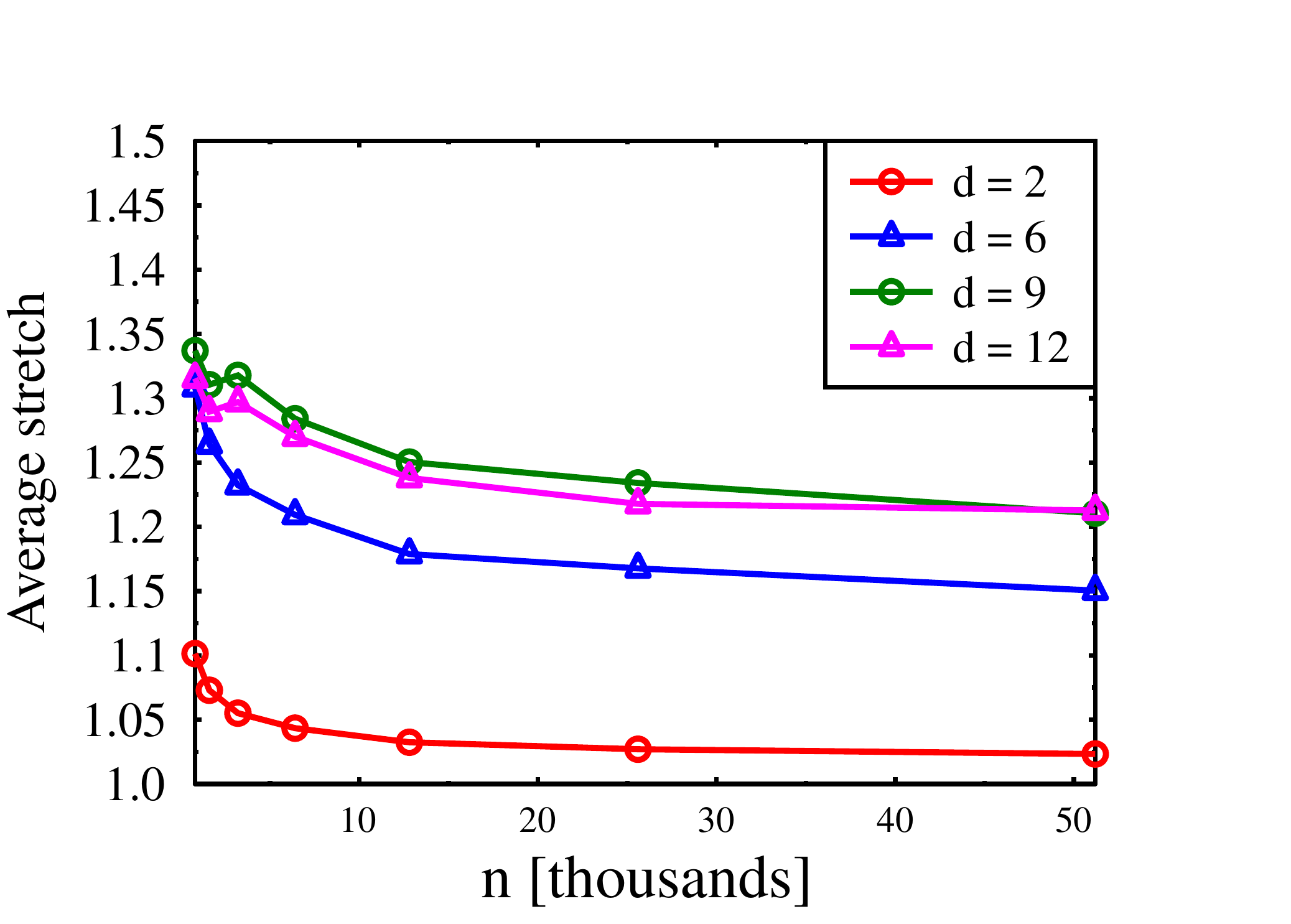}
		\label{fig:prm_stretch_obs} 
	}
	\caption{ \sf (a) Connectivity and (b) stretch for RGGs in the toy scenarios.
} 
	\label{fig:res_obs}
\end{figure}


\vspace{5pt}
\noindent \textbf{Motion-planning algorithms.}  Finally, we compare
PRM and Soft-PRM as sampling-based planners for rigid-body motion
planning on the Home scenario (Fig.~\ref{fig:home}) provided by the
OMPL distribution.\footnote{We used a robot which was scaled down to
  $80\%$ the size of the robot provided by the OMPL distribution.}
This six-dimensional configuration space, SE3, includes both
translational and rotational degrees of freedom.  Thus, it is not
clear if the theoretic results presented in this paper still hold in
this non-Euclidean space.

To apply the results, a key question one has to address is how to
choose the connection radius when using a non-Euclidean metric.  Let
$x$ be a point sufficiently far from the boundary and let
$r_n = \gamma \left( \log n / n \right)^{1/d}$ be the connection
radius used.  When using the Euclidean metric, the average number of
neighbors of $x$ is
$\text{nbr}(n) = \left( 2^{d - 1} / d \right) \cdot \log n$.  Thus,
for each value of $n$, we sampled 100 random points and, for each one,
computed the radius $r$ for which the point had $\text{nbr}(n)$
neighbors.  Finally, we used the average value over all such points in
the experiments.

Figure~\ref{fig:mp_res} presents the cost of the solution produced by
each algorithm as a function of the running time.  Similar to the
previous tests, both algorithms exhibit similar behavior, and the cost
obtained approaches the optimum as the number of nodes increases.
%

\begin{figure}[t]
  \centering 
  	\subfloat [\sf ] 
  	{
	    \includegraphics[trim=0.cm 0.0cm 1.5cm 1.5cm,height=3.4cm]{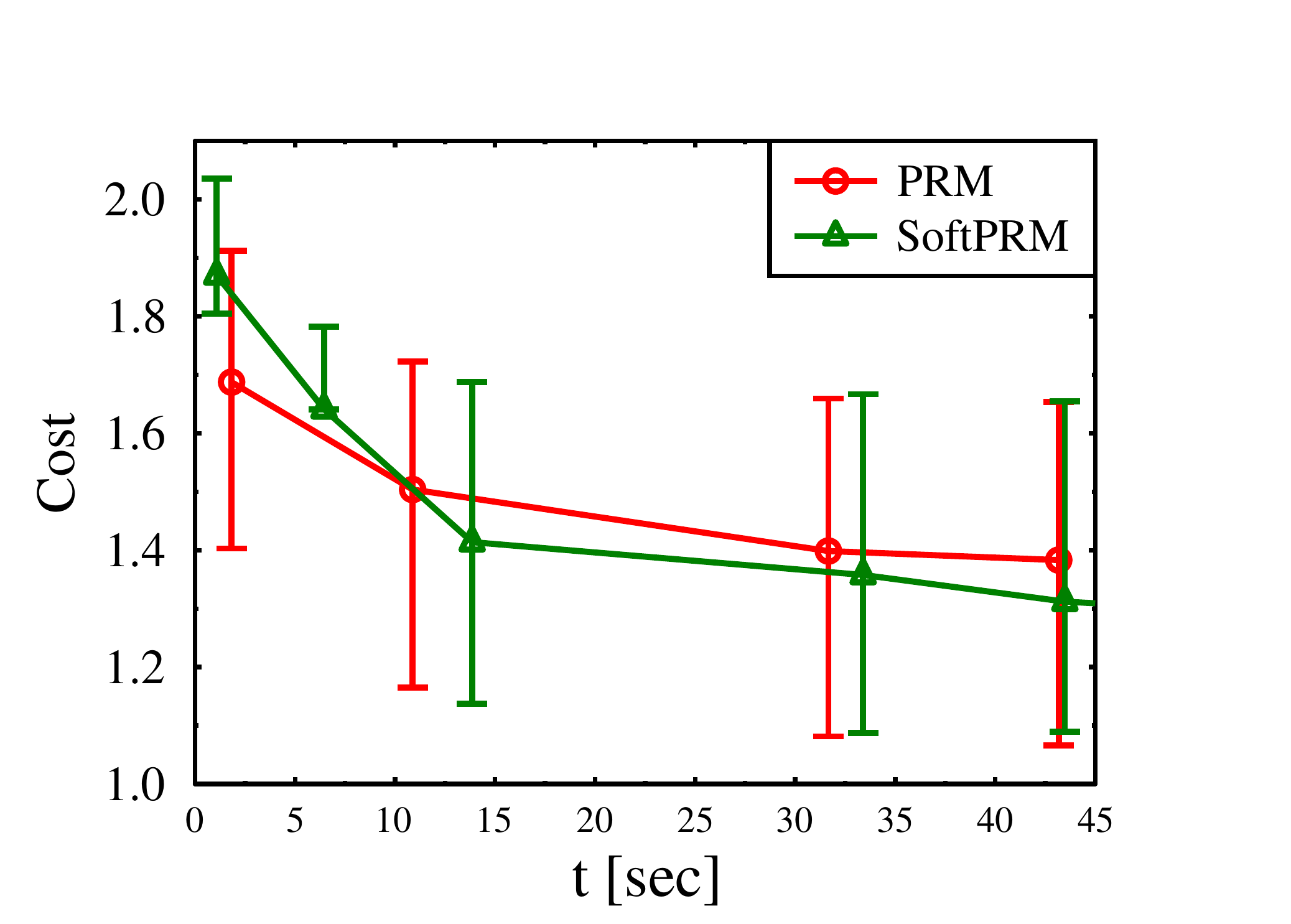}
		\label{fig:mp_res} 
	} 
	\subfloat [\sf ] 
	{
	    \includegraphics[height=3.3cm]{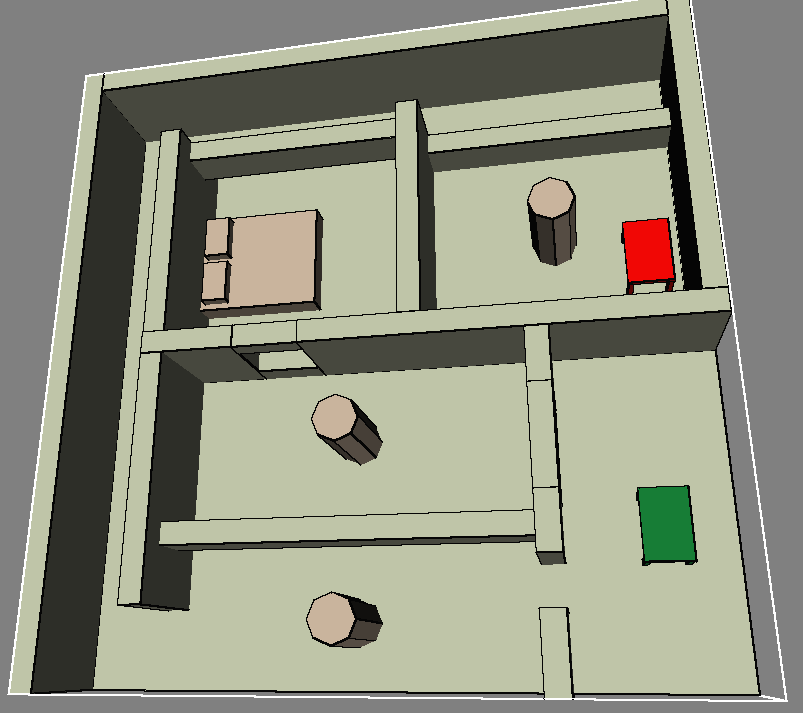}
		\label{fig:home} 
	}
	\caption{ \sf  (a) Average quality obtained by the PRM and Soft-PRM algorithms as a function of the running times in the Home scenario (b).
	Error bars in (a) denote the 20'th and 80'th percentile. Cost is
    normalized such that the unit cost represents the optimum. 
  }
	\label{fig:mp}
\end{figure}

\section{Discussion}\label{sec:discussion}
We conclude this paper by describing a connection between the
Bluetooth-PRM and Soft-PRM algorithms and approximate
\emph{nearest-neighbor (NN) search} in sampling-based motion
planning. We then proceed to describe future research directions that
follow from our work.\vspace{5pt}

\noindent\textbf{Approximate NN search in motion planning.}
NN search is a key ingredient in the implementation of sampling-based
planners (see, e.g., line~3 in Alg.~\ref{alg:soft}). Typically, exact
NN computation, where all the neighbors of a query point in a given
area are reported, tends to be slow in high dimensions, due to the
``curse of dimensionality''~\cite{HIM12}. Thus, most implementations
of motion planners involve approximate NN libraries (see,
e.g.,~\cite{AMNSW98,KSH15,flann}), which are only guaranteed to return
a subset of the neighbors of a given query point (see, e.g.,
\cite{AKS14, B75,FBF77,IM98}).

However, existing proofs of probabilistic completeness and asymptotic
optimality of standard planners (see, e.g.,~\cite{JSCP15, KF11,
  KSLO96}) assume that NNs are computed exactly. Without these
assumptions, the proofs no longer hold (although it may be possible to
modify them to take this into account).

The analysis given in this paper bridges this gap: PRM, when
implemented with approximate NN search, can be modeled as a
Bluetooth-PRM or Soft-PRM.  Thus, the former algorithm is
probabilistically complete by using the probabilistic completeness of
the latter (see Theorem~\ref{thm:complete_final}).  \vspace{5pt}

\noindent \textbf{Future work.} 
The literature of RGGs is rich and encompasses many models which were
not addressed in this work due to lack of space (see,
e.g.,~\cite{BroETAL14, BBSW09, Wade09}).  Such models can be used to
analyze existing planners and might lead to the development of novel
planners.

In this work our focus was on Euclidean configuration spaces and the
standard Euclidean distance. We mention that several works on RGGs
consider different metrics in the Euclidean space (see,
e.g.,~\cite{AppRus02, Pen99}). Such results can be imported to the
setting of motion planning using our framework, with slight
modification of the proofs.

Perhaps a more urgent issue involves the analysis of exiting planners
in complex configuration spaces. To the best of our knowledge the
behavior of standard planners such as PRM and RRT* is not well
understood for non-Euclidean spaces, even for the simple case of a
rigid-body robot translating and rotating in a three-dimensional
workspace. We believe that several results involving RGGs in complex
domains can shed light on this question.  For instance,
Penrose~\cite{Pen97} considers the case where points are sampled on a
torus, whereas Penrose and Yukich~\cite{PenYuk13} study the setting of
points on a manifold embedded in Euclidean space.

\appendix

\section*{Appendix: Connectivity of SRGGs}
Recall that an SRGG $\G_n:=\Gsoft(\X_n;r_n;\phi_n)$ represents a graph
such that for every two $x,y\in \X_n$, there is an edge in $\G_n$ with
probability $\phi_n (\|x-y\|_2)$ if $ \|x-y\|_2 \leq r_n$.  We refer
to $\phi_n:\dR^+ \rightarrow [0,1]^d$ as the \emph{connection
  function}. 

This appendix is devoted to the proof of
Theorem~\ref{thm:srgg_con}. Namely, we show that for
\[r_n=\rtrs,\quad \gamma >(d+1)^{1/d}\gamma^*\]
and for every $z\in \dR^+$ the connection function is defined to be
\[
\phi_n(z)= \begin{cases}
  0  & \quad \textup{if }z > r_n,\\
  1-\frac{z}{r_n}  & \quad \textup{if }z\leq r_n,\\
\end{cases}\]
the SRGG $\G_n$ is connected \as For the simplicity of presentation we
will use the notation $r:=r_n, \phi:=\phi_n$, which will be fixed
throughout the text. Furthermore, we assume that $n$ is large enough
so that $r<1/2$.

A ingredient of the proof is the following statement by
Penrose~\cite{Pen13}. Define 
\[ I_n = n \int_{x \in [0,1]^d} \exp \left( -n \int_{y \in [0,1]^d}
  \phi(y-x)\dy \right)\dx.\]
If $\limn I_n = 0$ then $\G_n$ is connected \as Thus, it suffices to
show that $I_n$ tends converges to $0$.

As is often the case with random geometric graphs (see
e.g.~\cite{Pen03, Pen13}), we will have to carefully consider boundary
effects, which result from samples that lie close to the boundary of
$[0,1]^d$.  Thus, before computing $I_n$ for our connection functions,
we introduce some notations that will allow us to handle such cases
more easily.

Recall that the $d$-dimensional ball of radius $r$, centered at
$x\in \dR^d$ is denoted by
$\B_r(x)=\left\{y\in \dR^d\middle| \|x-y\|_2\leq r\right\}$.
Given $x\in [0,1]^d$ define
$\tilde{\B}_{r}(x)=\B_{r}(x)\cap [0,1]^d$. Additionally, set
$\nu_d = \theta_d r^d$ to be the volume of $\B_r$.

We subdivide $[0,1]^d$ into the regions $R_0 \ldots R_d$: for
$0\leq j\leq d$ let $R_j$ to be the set of points for which $B_r(x)$
intersects the boundary of $[0,1]^d$ along exactly $j$ axis-aligned
hyper-planes. We bound the volume of $\tilde{\B}_r(x)$ according to
the region $R_j$ such that $x\in R_j$. For instance, when $d=2$, the
regions $R_0, R_1, R_2$ are the set of points $x$ where the disk
$B_r(x)$ does not intersect $[0,1]^d$ at all, intersects $[0,1]^d$
along exactly one line, or intersects $[0,1]^d$ along two lines,
respectively (see Figure~\ref{fig:regions}).

\begin{claim}\label{clm:volume}
  For every $0\leq j\leq d, x \in R_j$ we have that
  $\nu_d / 2^{j} \leq |\tilde{\B}_r(x)|$.
\end{claim}

\begin{proof}
  $\tilde{\B}_r(x)$ attains the minimal volume when $x$ is located on
  the intersection of exactly $j$ $d$-dimensional hyperplanes that
  form the boundary of $[0,1]^d$. Every such hyperplane 
  cuts the volume of $\nu_d$ in half. Thus,
  $\nu_d / 2^{j} \leq |\tilde{\B}_r(x)|$.
\end{proof}

\begin{claim}\label{clm:regions}
  For every $0\leq j\leq d$ we have that $|R_j| \leq c_j r^j$, where
  $c_j$ is some positive constant.
\end{claim}

\begin{proof}
  We fix a specific $j$. The number of maximal connected components of
  $R_j$ only depends on $d$. Denote this number by $c_j$.  The volume
  of each such component can be expressed as $(1-2r)^{d-j}r^j$ since
  the number of $d$-dimensional hyperplanes that are in contact with
  the $r$-radius ball is exactly~$j$. Thus,
  $|R_j|=c_j(1-2r)^{d-j}r^j\leq c_jr^j$.
\end{proof}

\begin{figure*}[t]
  \centering 
   	\includegraphics[width=0.5\textwidth]{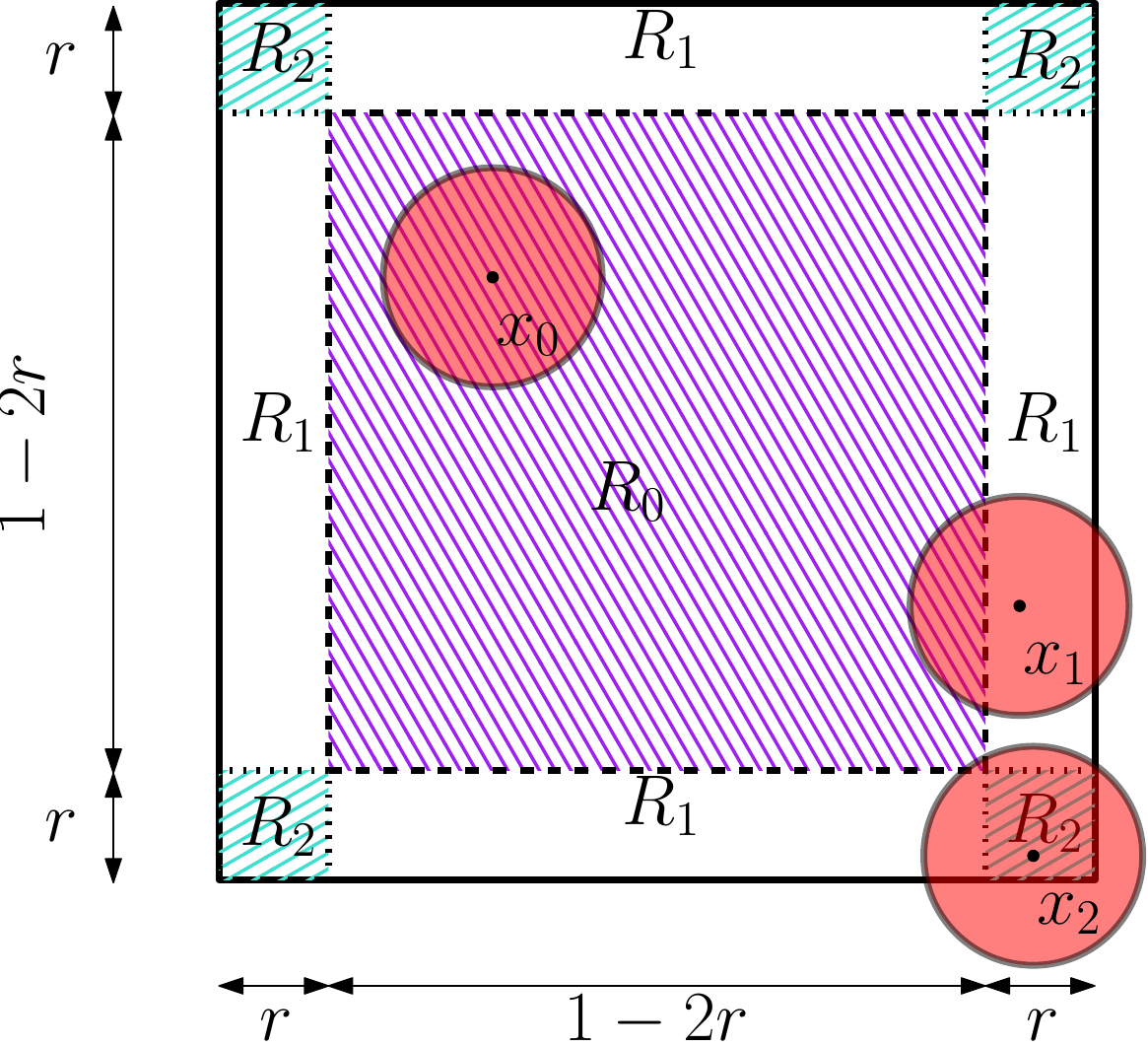}
	\label{fig:region} 
	\caption{ \sf Visualization of regions used in proof of
      Theorem~\ref{thm:srgg_con} for the planar case. The regions
      $R_0, R_1$ and $R_2$ are depicted by blue, white and cyan
      regions. For each region $R_j$ one point $x$ is shown together
      with a red disk surrounding it. Notice that the disk intersects
      exactly $j$ lines supporting the unit cube $[0,1]^2$.}
\end{figure*}

\begin{claim}\label{clm:y_int}
  For every $0\leq j\leq d, x\in R_j$ we have that 
  \[\int_{y \in [0,1]^d} \phi(y-x)\dy\geq \frac{\nu_d}{(d+1)2^j}.\]
\end{claim}
\begin{proof} 
  We use the notation
  $\S_r(x)= \left\{y\in \dR^d\middle| \|x-y\|_2=r\right\}$ for the
  $d$-dimensional $r$-sphere.

  \begin{align*}
    \int_{y \in [0,1]^d} \phi(y-x)\dy &  = \int_{y \in
                                        \tilde{\B}_r(x)}  \phi(y-x)\dy
    \\ & \stackrel{(1)}{\geq} \frac{1}{2^j}\int_{y \in
         \B_r(x)}  \phi(y-x)\dy \\ & = \frac{1}{2^j}\int_{y \in
                                     \B_r(x)}
                                     \frac{r-\|y-x\|_2}{r}\dy
    \\  & = \frac{1}{2^j}\left(\int_{y \in
          \B_r(x)}\dy
          - 
          \frac{1}{r}\int_{y \in
          \B_r(x)}\|y-x\|_2\dy\right)
    \\  & = \frac{1}{2^j}\left(\nu_d
          - 
          \frac{1}{r}\int_{y \in
          \B_r(x)}\|y-x\|_2\dy\right)
    \\  & \stackrel{(2)}{=} \frac{1}{2^j}\left(\nu_d
          - 
          \frac{1}{r}\int_{\rho\in [0,r]}|\S_\rho(x)|\rho\drho\right)
    \\  & \stackrel{(3)}{=} \frac{1}{2^j}\left(\nu_d
          - 
          \frac{1}{r}\int_{\rho\in
          [0,r]}\frac{d}{\rho}|\B_\rho(x)|\rho\drho\right) 
    \\  & = \frac{1}{2^j}\left(\nu_d
          - 
          \frac{d}{r}\int_{\rho\in [0,r]}|\B_\rho(x)|\drho\right) 
    \\  & \stackrel{(4)}{=} \frac{1}{2^j}\left(\nu_d
          - 
          \frac{d}{r}\cdot \frac{r}{d+1}\nu_d\right)  
    \\ & = \frac{\nu_d}{(d+1)2^j}.
  \end{align*}
  Explanation for the non-trivial transitions:
  (1)~Claim~\ref{clm:volume} and the fact that $\phi$ is
  non-negative. (2)~Changing integrating parameters: instead of
  integrating over all distances $\|x-y\|_2,$ we integrate over all
  radii $\rho \in [0, r]$.  For each such radius, we multiply the
  volume of the sphere $\S_\rho(x)$ by $\rho$. (3)~The relation that
  for every dimension $d$ we have
  $|S_\rho(x)| = \frac{d}{\rho} |\B_\rho(x)|$. (4)~The fact that
  $|\B_\rho(x)|=c\rho^d$ for some constant $c>0$, which reduces the
  integration to a polynomial.
\end{proof}

We are now ready to tame the beast.

\begin{proof}[Proof of Theorem~\ref{thm:srgg_con}]
  Recall that it suffices to show that $I_n$ converges to $0$. Then
  \begin{align*}
  	I_n &= n \int_{x \in [0,1]^d} 
          \exp 
          \left( 
          -n \int_{y \in [0,1]^d} \phi(y-x)\dy  	
          \right)
          \dx\\
        & = n\sum_{j=0}^d\int_{x\in R_j} \exp 
          \left( 
          -n \int_{y \in [0,1]^d} \phi(y-x)\dy  	
          \right)
          \dx \\ 
        & \stackrel{(*)}{\leq}
          n\sum_{j=0}^d\int_{x\in R_j} \exp 
          \left( 
          -n \cdot \frac{\nu_d}{(d+1)2^j}
          \right)
          \dx\\ 
        & =
          n\sum_{j=0}^d |R_j| \exp 
          \left( 
          -n \cdot \frac{\nu_d}{(d+1)2^j}
          \right)\\ 
        & =
          n\sum_{j=0}^d |R_j| \exp 
          \left( 
          -n \cdot \frac{\theta_d r_n^d}{(d+1)2^j}
          \right)\\
        & =
          n\sum_{j=0}^d |R_j| \exp 
          \left( 
          -n \cdot \frac{\theta_d}{(d+1)2^j}
          \cdot \frac{\gamma^d\log n}{n}\right)\\
        & =
          n\sum_{j=0}^d |R_j| \exp 
          \left( 
          -\frac{\theta_d\gamma^d}{(d+1)2^j}
          \cdot \log n\right) ,
     \end{align*}
     where $(*)$ is due to Claim~\ref{clm:y_int}. Denote
     $a:=\frac{\theta_d\gamma^d}{(d+1)2^j}$. We have that
     \begin{align*}
       I_n & \leq n\sum_{j=0}^d |R_j| \exp 
             \left( 
             -a\log n\right) =  n\sum_{j=0}^d |R_j| n^{-a}\\
           & \stackrel{(**)}{\leq} n \sum_{j = 0} ^d  c_j r^j n^{-a} =
             n\sum_{j=0}^d c_j \gamma^j (\log n)^{j/d}n^{-j/d}n^{-a}
       \\ & =
            \sum_{j=0}^d c_j\gamma^j n^{1-j/d-a} (\log n)^{j/d},
     \end{align*}
     where $(**)$ is due to Claim~\ref{clm:regions}.  We now show that
     for $\gamma >(d+1)^{1/d}\gamma^*$ we have that $j/d+a>1$, which
     implies that that $I_n$ converges to $0$.
     \begin{align*}
       \frac{j}{d}+a & = \frac{j}{d} +
                       \frac{\theta_d\gamma^d}{(d+1)2^j} > \frac{j}{d} +
                       \frac{\theta_d(d+1)(\gamma^*)^d}{(d+1)2^j}\\ & = \frac{j}{d} +
                                                                      \frac{\theta_d2^d}{2d\theta_d2^j} =
                                                                      \frac{j+2^{d-j-1}}{d}\geq 1.
     \end{align*}
\end{proof}

\bibliographystyle{abbrv}
\bibliography{rgg}

\end{document}